\documentclass[11pt,letterpaper]{amsart}
\usepackage{preamble}
\counterwithin{equation}{section}
\begin{document}

\title[ Limitations of t-SNE]{Some Theoretical Limitations of t-SNE}

\author[Rupert Li]{Rupert Li}
\address[]{Stanford University, Stanford, CA 94305, USA}
\email{rupertli@stanford.edu}

\author[Elchanan Mossel]{Elchanan Mossel}
\address[]{Massachusetts Institute of Technology, Cambridge, MA 02139, USA}
\email{elmos@mit.edu}

\begin{abstract}
t-SNE has gained popularity as a dimension reduction technique, especially for visualizing data.
It is well-known that all dimension reduction techniques may lose important features of the data.
We provide a mathematical framework for understanding this loss for t-SNE by establishing a number of results in different scenarios showing how important features of data are lost by using t-SNE. 
\end{abstract}

\maketitle

\section{Introduction}\label{section:introduction}
Dimension reduction is often a central step in the visualization and analysis of high-dimensional data.
In particular, our inability to visualize in more than $2$ or $3$ dimensions often leads researchers to present their data in $2$ or $3$ dimensions in order to identify qualitative structure in the data, such as clusters.
Dimension reduction is also often used as part of computational pipelines in data analysis. 
Many dimension reduction techniques exist, but t-distributed stochastic neighborhood embedding (t-SNE), introduced by van der Maaten and Hinton \cite{maaten2008visualizing} in 2008, has become perhaps the most popular for data visualization.
t-SNE is a nonlinear algorithm, and this nonlinearity enables it to succeed in many settings where more classical, linear techniques spectacularly fail.
For example, see \cref{fig:simplex} for a comparison of principal component analysis (PCA) and t-SNE on data generated from a high-dimensional simplex.

\begin{figure}[htbp]
   \centering
   \begin{subfigure}[t]{0.48\textwidth}
       \centering
       \includegraphics[width=\textwidth]{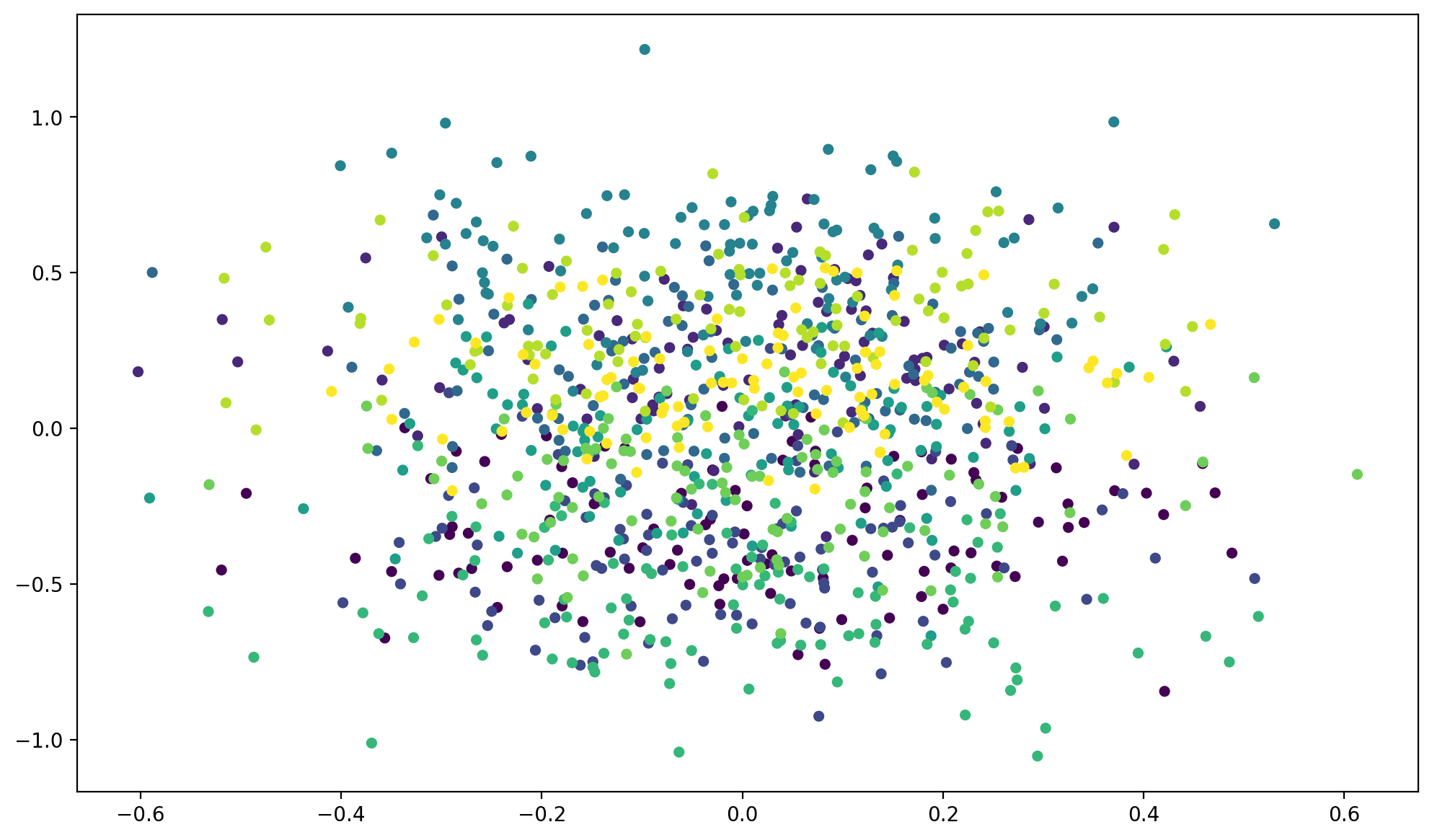}
       \caption{First two principal components}
   \end{subfigure}
   \hfill
   \begin{subfigure}[t]{0.48\textwidth}
       \centering
       \includegraphics[width=\textwidth]{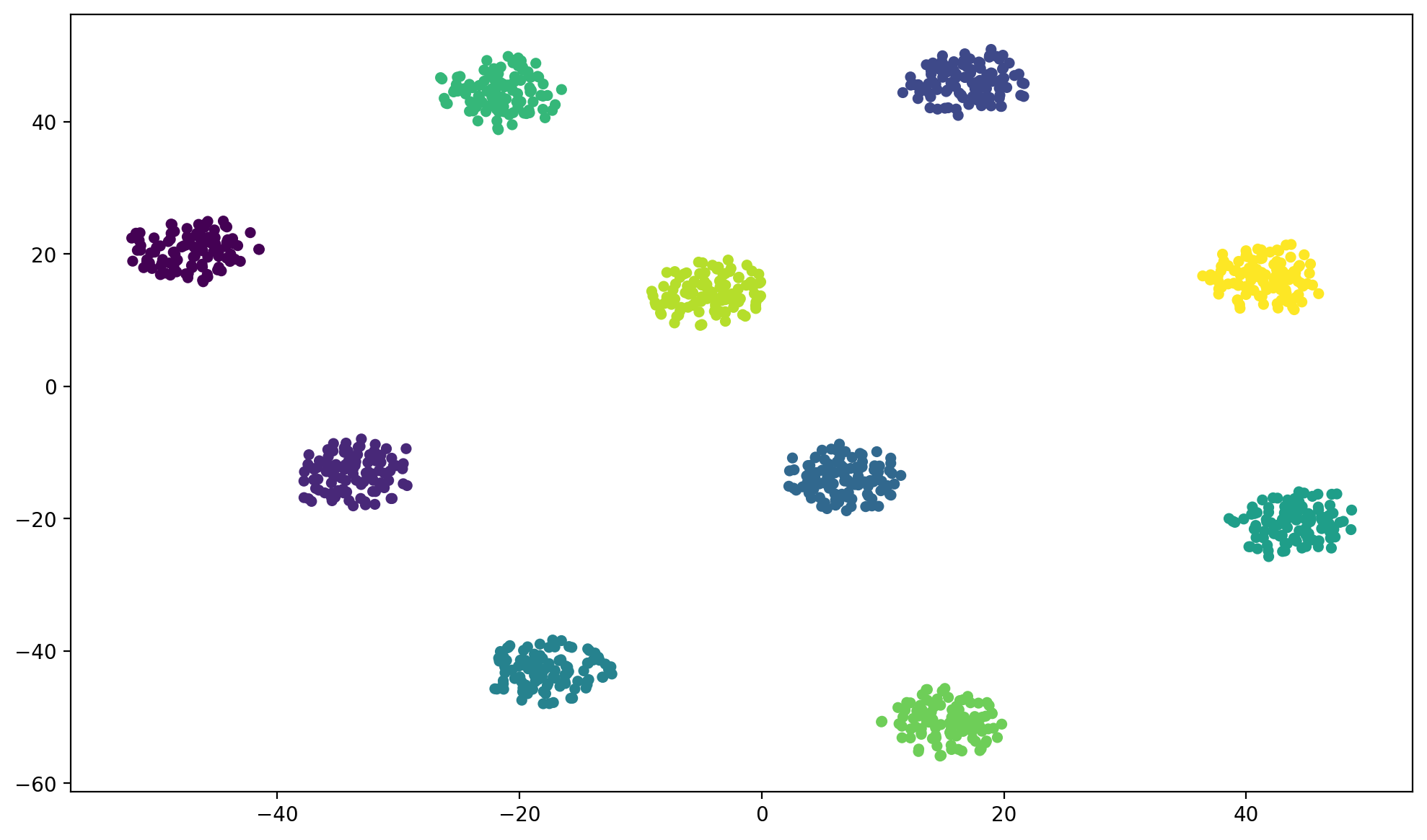}
       \caption{t-SNE output using standard parameters}
   \end{subfigure}
   \caption{t-SNE and PCA on points clustered around the 10 vertices of a regular 9-simplex. Here, the vertices are the ten elementary basis vectors in $\R^{10}$, and each cluster has 100 points sampled i.i.d.\ from a spherical Gaussian with standard deviation $0.2$ in each direction.}
   \label{fig:simplex}
\end{figure}

We refer readers to \cref{section:tsne} for a detailed description of the t-SNE algorithm. 
In essence, t-SNE computes a certain similarity measure between all pairs of points in the high-dimensional dataset, and then considers a candidate mapping of this dataset to points in a lower-dimensional space, typically 2- or 3-dimensional.
It then computes a different similarity measure between these low-dimensional points, and attempts to arrange these points so that these two similarity measures are as close to each other as possible.
It does so by minimizing the KL-divergence between these two matrices of similarity measures.

We first discuss our results and then how they relate to some prior theoretical analysis of t-SNE~\cite{tSNE_provably,cai2022theoretical,tSNE_visualization,tSNE_iid}. 

\subsection{Our Results}\label{subsec:our_results}
One of the intuitive reasons for the challenge in dimension reduction is that higher dimensions have more ``space" to pack data.
This may result in data points that were far away in the original data while their low dimensional projections are close together.
More formally, we have the following standard result for any dimension reduction procedure:

\begin{proposition}\label{proposition:volume}
    Let $X$ consist of $n$ points sampled uniformly from $\SS^{d-1}$, with $n \leq 2^d$.
    Let $g=g(n)$ be any sequence tending to infinity.
    With high probability as $n,d\to\infty$, for any map $f:X\to\R^2$ with $\norm{f(x)}\leq B$ for all $x\in X$, for $1-O(g^{-2})$ fraction of the points $x\in X$, there exists another point $y\in X$ such that $\norm{x-y}\geq0.2$ yet $\norm{f(x)-f(y)}\leq 3B\frac{g}{\sqrt{n}}$.
\end{proposition}

We now turn our attention specifically to t-SNE.
The goal of the t-SNE procedure is to map the data points to low dimension so that their normalized pairwise distances in the high dimensions $P$ is close in KL-divergence to the normalized pairwise distances $Q$ of the images with a different kernel.  
We next show that even for simple collections of points, it is impossible to find a $Q$ such that $D(P \Vert Q)$ is small. 
Throughout this paper unless otherwise stated, we assume the t-SNE output is two-dimensional, though all our theoretical results straightforwardly extend to any fixed output dimension.
\begin{proposition}\label{proposition:doubled_frame}
    Let $X = \set{e_1, \dots, e_n, 2e_{n+1}, \dots, 2e_{2n}} \subset \R^{2n}$ be an orthogonal frame with half its points doubled in magnitude.
    Assume fixed bandwidths $\sigma_i = \sigma$.
    Then there exists a constant $c > 0$ such that for all sufficiently large $n$ and any embedding $Y$, the t-SNE objective satisfies $D(P\Vert Q) \geq c$.
\end{proposition}

It is well-known that the number of equidistant points in dimension $d$ is at most $d+1$.
It is natural to ask how t-SNE reduces the dimension of such collections of points.
The following proposition provides an answer to this question in terms of the global optimizer of the t-SNE objective: 

\begin{proposition}\label{proposition:single_point}
    Consider t-SNE on a dataset $X$ of $n$ points with $s$-dimensional output such that $n>s+1$.
    If $X$ is an equidistant set, i.e., all pairwise distances are equal, or the $\sigma_i$ are set using a perplexity value $p$ and we have $p\geq n-1$, then any global minimizer $Y$ of the t-SNE objective $D(P\Vert Q)$ has all points $y_1,y_2,\dots,y_n$ coinciding at the same point.
\end{proposition}

Arguably the results above show a failure mode of t-SNE.
However, it could be argued that it is a tailored example that does not represent the generic situation.
The result below exhibits the same phenomenon in a generic setting.

\begin{theorem}\label{theorem:sphere}
    Consider an asymptotic regime as $n,d\to\infty$, where the bandwidths $\sigma_1,\sigma_2,\dots,\sigma_n$ all equal some fixed constant $\sigma$.
    Fix constant $\delta\in(0,1/2)$ and $r=\omega\paren{d^{-1/6+\delta/3}}$.
    For some $n=e^{\Theta\paren{d^{2\delta}}}$, letting $x_1,x_2,\dots,x_n$ be uniformly random samples on the unit sphere $\SS^{d-1}$, with probability $1-e^{-\Omega\paren{d^{2\delta}}}$, any global minimizer $Y$ of the t-SNE objective $D(P\Vert Q)$ has $1-r$ fraction of $Y$ within some (open) ball of radius $r$ and has all of $Y$ contained within some ball of radius $O(\sqrt{r})$.
\end{theorem}

Informally, the theorem states that the t-SNE objective will be optimal only when it maps all the points to a tiny neighborhood. 
Here, note that larger $\delta$ lets us have more points on the sphere, and gives us a higher probability of success; however, it gives a weaker result to the extent that $r$ is bigger, so fewer of the points are within the radius-$r$ ball and this ball as well as the radius-$O(\sqrt{r})$ ball are both bigger.

This could be considered a failure mode of t-SNE: the dataset $X$ is too high-dimensional, so that to capture the structure of many points in $X$ being approximately equidistant to each other in different directions, the best solution for t-SNE is to put all points at essentially the same point, yielding an uninformative dimension reduction. Additionally, if the visualization is at such small scale, there may be additional concerns regarding numerical stability of the algorithm.

Moreover, we can apply \cref{proposition:volume} to the setting of \cref{theorem:sphere} with zoom-in.
For this, let $n=e^{\Theta\paren{d^{2\delta}}}$ from \cref{theorem:sphere}, which is eventually less than $2^d$, and let $B$ be the minimum possible radius of a closed ball containing all of $Y$, which by \cref{theorem:sphere} is $O\paren{\sqrt{r}}=\omega\paren{d^{-1/12+\delta/6}}$ with high probability, so $B$ can be taken to be $O\paren{d^{-\varepsilon}}$ for any constant $\varepsilon>0$ satisfying $\varepsilon<\frac{1}{12}-\frac{\delta}{6}$.
Let $g$ slowly tend to infinity, say $g=d$.
Note then that $\frac{g}{\sqrt{n}}=e^{-\Theta\paren{d^{2\delta}}}$.
Then \cref{proposition:volume} implies that with high probability, even when we zoom in to the ball of radius $B$, so at the proper relative scale for $Y$, we have $1-o(1)$ fraction of the points have another point within relative distance $\frac{3g}{\sqrt{n}}=e^{-\Theta\paren{d^{2\delta}}}=o(1)$, despite all pairs of points in $X$ having constant order distance, namely in $[0.2,2]$.
So t-SNE will put all points in a very small ball, and even within this ball, most points will be very close to some other point that it was not originally close to in $X$.

\subsection{Related Work}
The theoretical analysis of t-SNE is limited and most of the analysis is for parameter settings that are different from the standard conventions used in practice.
Linderman and Steinerberger \cite{tSNE_provably} prove that when t-SNE is in the early exaggeration stage with $\alpha=\Theta(n)$ and $h=\Theta(1)$, where $n$ is the number of datapoints, the diameters of embedded clusters decay exponentially under the gradient descent updates until they reach some size depending on how well the dataset $X$ was originally clustered, with better-separated data yielding tighter clusters.
They note that in practice, standard choices of $\alpha$ and $h$ are constants 12 and 200, respectively; this distinction between constant $\alpha$ being used in practice versus linear $\alpha=\Theta(n)$ for theoretical results on t-SNE is shared among the other theoretical papers on t-SNE. 
They made the observation that with such large early exaggeration coefficient $\alpha$, the t-SNE algorithm behaves like a spectral clustering algorithm; Cai and Ma \cite{cai2022theoretical} formally proved this observation.
Linderman and Steinerberger show the clusters shrink exponentially, and provide a heuristic argument though no formal proof explaining why these clusters should be disjoint, i.e., separated.
In contrast, Arora, Hu, and Kothari \cite{tSNE_visualization} do not consider the dynamical behavior but do show that for spherical and well-clustered data, with similar parameter choices, namely including the case $\alpha=\Theta(n)$ and $h=1$, t-SNE in the early exaggeration stage will output a visualization of $X$ in which the clusters are separated.
Note that the Linderman--Steinerberger definition of well-clustered data and the Arora--Hu--Kothari definition of well-clustered data are different, though are of a similar spirit, so in practice it seems reasonable to assume that for well-clustered data, both conclusions will hold: the clusters will shrink exponentially quickly until they reach some desired size, and at this size the clusters are disjoint.
Arora, Hu, and Kothari also study the interesting example of a mixture of two concentric spherical Gaussians with different radii, i.e., variances, proving that t-SNE again in early exaggeration with $\alpha=\Theta(n)$ can partially recover the inner Gaussian, meaning there's a cluster that mostly consists of points from the inner Gaussian, though were unable to provide any guarantees about the outer Gaussian.
This example is structurally similar to our doubled frame counterexample in \cref{proposition:doubled_frame}.
Auffinger and Fletcher \cite{tSNE_iid} prove a result of a significantly different flavor, showing that if the datapoints are drawn i.i.d.\ from a compactly supported distribution with continuously differentiable density function, then t-SNE will converge to some null distribution as $n\to\infty$, which also has compact support.

\subsection{Relation between our Results and \cite{tSNE_visualization}}\label{subsec:relation_AHK}
We want to compare the negative results of \cref{theorem:sphere} to the positive results of \cite{tSNE_visualization}.
As discussed in greater detail in \cref{section:split_sphere}, let $\delta=0.49$ and consider a dataset $X$ obtained by drawing $n=e^{\Theta\paren{d^{2\delta}}}$ uniformly random samples from the unit sphere $\SS^{d-1}$, as in the setting of \cref{theorem:sphere}, except we discard points lying in an equatorial region, e.g., we discard points whose first coordinate has magnitude less than $d^{-0.1}$.
This region is thin enough that we still have $e^{\Theta\paren{d^{0.98}}}$ points with high probability, so that the asymptotic growth of $n$ is essentially unchanged.
Then the conclusions of \cref{theorem:sphere} and \cref{proposition:volume} continue to hold for this ``split sphere'' dataset $X$.
However, under the well-known approximation that the coordinates of these random points from the sphere $\SS^{d-1}$ are asymptotically Gaussian as $d\to\infty$, of course aside from the first coordinate which concentrates around $\pm d^{-0.1}$, this $d^{-0.1}$ separation between the two spherical caps that comprise our split sphere $X$ is large enough that one should expect from the positive result of Arora, Hu, and Kothari \cite[Corollary 3.2]{tSNE_visualization} on isotropic Gaussian mixtures that t-SNE successfully distinguishes the two spherical caps, i.e., outputs a visualization where two clusters, corresponding to the two spherical caps, can be drawn, such that they are separated from each other on a relative scale.
The difference between our negative result and the positive results in~\cite{tSNE_visualization} 
can be interpreted in a number of ways including the following: 
\begin{enumerate}
    \item First, it's possible that both behaviors are exhibited: t-SNE puts all points of $X$ very close together, but zooming in to this relative scale, t-SNE does at least separate the two spherical caps into different clusters, yet within each cluster still fails to capture local structure as most points are close to some other point that it is not originally close to in $X$.
    \item Second, because our results consider the global minimizer of the t-SNE objective $D(P\Vert Q)$, while the results of Arora, Hu, and Kothari \cite{tSNE_visualization} consider the t-SNE output after the early exaggeration phase, without running additional iterations after early exaggeration to converge to a local minimum of $D(P\Vert Q)$, it's possible that the intermediate t-SNE output looks good immediately after the early exaggeration phase, but is ruined by the subsequent iterations without early exaggeration.
    \item Third, if one believes the intermediate t-SNE output immediately after the early exaggeration phase is not significantly different from the final t-SNE output, then instead of the previous two interpretations, one could interpret the discrepancy between these sets of results as arising from a difference between the local minimum that the t-SNE algorithm obtains and the global minimum of the t-SNE objective $D(P\Vert Q)$ that we consider. 
    In other words, it's possible t-SNE finds a local minimum of $D(P\Vert Q)$ that displays a good visualization separating the two clusters, but the global minimum displays the undesirable behavior that we described.
    
    This could suggest that the t-SNE objective $D(P\Vert Q)$ may not be a good objective function to consider, yet somehow the t-SNE output miraculously works by finding a good local minimum.
    This would raise many interesting questions about the dynamics of the t-SNE algorithm, as well as prompt further study regarding how to converge to a good local minimum; one could imagine using tools from the machine learning literature, such as stochastic gradient descent, to arrive at different local minima.
    Note that $D(P\Vert Q)$ is convex in $Q$, so would have a unique minimizer for $Q$, but this does not mean $D(P\Vert Q)$ is convex in $Y$, and the minimizers of $Y$ are trivially not unique as applying any isometry on the output space $\R^s$ yields another minimizer $Y'$.
    
    Also, we note that one would need to believe the premise of this third interpretation, that the intermediate and final t-SNE outputs are similar, in order for the results of \cite{tSNE_provably,tSNE_visualization} to apply to practical implementations of t-SNE, where early exaggeration is only used for the initial, i.e., ``early'', iterations of the algorithm.
\end{enumerate}
Numerical examples from \cref{section:numerical_examples} primarily support the first interpretation, but we note the interpretations are not mutually exclusive and the dimensionality of our examples is limited so may not capture the full asymptotic picture.  
For example, one observation from our experiments in \cref{section:numerical_examples} is that the output immediately after early exaggeration is significantly different from the output after the entire algorithm concludes, i.e., after many iterations without early exaggeration have occurred.
Moreover, early exaggeration exhibits tighter clusters, at the cost of a more degenerate visualization within each cluster, lending credence to the second interpretation, i.e., a scenario where early exaggeration separates the clusters, but when the clusters expand and rearrange after early exaggeration, the separation is lost.

\subsection{Paper Outline}

In \cref{section:tsne}, we provide a brief but comprehensive description of the t-SNE algorithm.
\cref{section:sphere} is our primary theoretical section, where we prove all the results stated in \cref{subsec:our_results}.
Then in \cref{section:split_sphere} we elaborate on \cref{subsec:relation_AHK} and explain how our results continue to hold for the split sphere example, and explain the heuristic by which we expect the positive result of Arora, Hu, and Kothari \cite{tSNE_visualization} to also hold for the split sphere.
In \cref{section:numerical_examples}, we present some numerical examples of the sphere and split sphere datasets.

\section{t-SNE}\label{section:tsne}
Suppose the input dataset $X$ consists of $n$ points with $d$ dimensions, and enumerate the datapoints $X=\set{x_1,\dots,x_n}\subset\R^d$.
We consider a probability distribution $P$ over unordered pairs of datapoints.
Specifically, for distinct $i,j\in[n]=\set{1,\dots,n}$, we define
\[ p_{j|i} = \frac{\exp\paren{-\frac{\norm{x_i-x_j}^2}{2\sigma_i^2}}}{\sum_{k\neq i}\exp\paren{-\frac{\norm{x_i-x_k}^2}{2\sigma_i^2}}} \]
to be a probability distribution $P_i$ over $[n]\setminus\set{i}$, corresponding to the points other than $x_i$, given by a spherical Gaussian kernel of variance $\sigma_i^2$, i.e., having covariance matrix $\sigma_i^2 I_d$, centered at $x_i$ and applied to the other points $x_j$.
Then $P$ is given by, for unordered pair $\set{i,j}$, which we denote as $ij\in\binom{[n]}{2}$,
\[ p_{ij} = \frac{p_{i|j}+p_{j|i}}{n}. \]
The bandwidth $\sigma_i$ is often chosen via binary search such that the \emph{perplexity} of $P_i$ equals (up to some tolerance) a user-defined value $p$.
Recall that perplexity is the exponential of the (discrete) Shannon entropy, so each $\sigma_i$ would be chosen such that
\begin{equation}\label{equation:perplexity}
    \log(p) = H(P_i) = -\sum_{j\neq i}p_{j|i}\log(p_{j|i}),
\end{equation}
where $H(P_i)$ denotes the Shannon entropy of $P_i$.
Note that if $\sigma_i\to\infty$, we have $p_{j|i}\to\frac{1}{n-1}$ for all $j\neq i$, so $H(P_i)\to\log(n-1)$.
Thus, if $p\geq n-1$, the binary search will attempt to set $\sigma_i\to\infty$ so that $P_i$ is uniform to maximize $H(P_i)$ at $\log(n-1)\leq\log(p)$.
For our purposes, where we do not think about the specific implementation and parameters of the binary search numerical algorithm and instead assume $\sigma_i$ is chosen so that \eqref{equation:perplexity} exactly holds, if $p\geq n-1$ we will use the convention that $\sigma_i=\infty$ and $P_i$ is uniform.
See \cref{remark:perplexity_too_high} for further discussion of the output of t-SNE in this case.

Meanwhile, we consider an $s$-dimensional embedding $Y$ of $X$, which we denote by $Y=\set{y_1,\dots,y_n}\subset\R^s$.
Typically t-SNE is used with $s=2$ or occasionally $s=3$, for data visualization purposes.
For any candidate $Y$, we consider an analogous probability distribution $Q$ over pairs $ij\in\binom{[n]}{2}$ given by a Cauchy kernel, i.e., a Student's $t$-distribution with one degree of freedom: $q_{ij}\propto\paren{1+\norm{y_i-y_j}^2}^{-1}$, or precisely
\[ q_{ij} = \frac{\paren{1+\norm{y_i-y_j}^2}^{-1}}{\displaystyle\sum_{k\ell\in\binom{[n]}{2}}\paren{1+\norm{y_k-y_\ell}^2}^{-1}}. \]
This is what gives t-SNE its name, as a successor to Stochastic Neighbor Embedding \cite{SNE}, which instead uses a Gaussian kernel for $Q$.

Then t-SNE finds the $Y$ that minimizes the Kullback--Leibler divergence between $P$ and $Q$,
\[ D(P\Vert Q) = \sum_{ij\in\binom{[n]}{2}}p_{ij}\log\frac{p_{ij}}{q_{ij}}, \]
using gradient descent.
It is straightforward to compute the gradient to be
\[ \frac{\partial}{\partial y_i}D(P\Vert Q) = 4\sum_{j\neq i}(p_{ij}-q_{ij})Z(y_i-y_j), \]
where $Z$ is the normalization constant
\[ Z = \sum_{ij\in\binom{[n]}{2}}\paren{1+\norm{y_i-y_j}^2}^{-1}. \]
In practice, a variety of modifications are occasionally made to the gradient update steps; we describe some of the most popular ones below.
The gradient can be decomposed into two sums,
\begin{equation}\label{equation:gradient}
    \frac{1}{4}\frac{\partial}{\partial y_i}D(P\Vert Q) = \sum_{j\neq i}p_{ij}q_{ij}Z(y_i-y_j) - \sum_{j\neq i}q_{ij}^2 Z(y_i-y_j),
\end{equation}
where the first term is referred to as the \emph{attractive} forces and the second term is referred to as the \emph{repulsive} forces.
Note that t-SNE minimizes $D(P\Vert Q)$ so moves $y_i$ opposite to its gradient, so the $j$th summand of the attractive force term, which is in the direction $y_i-y_j$, contributes to $y_i$ being updated in the opposite direction $y_j-y_i$, i.e., towards $y_j$, hence the terminology of attractive, and similarly repulsive.
One can scale the gradient update steps by a step-size $h>0$, typically a small constant, and more interestingly, t-SNE is often used with \emph{early exaggeration}, which multiplies the attractive forces by a constant $\alpha>1$ for some initial number of iterations, before returning to the original gradient, i.e., with $\alpha=1$, for the remainder of the iterative algorithm.
t-SNE is sometimes also implemented with momentum, so that each update step is derived from the gradient but also adds the previous update step, discounted via multiplying by some constant $\gamma\in(0,1)$.
Note that, assuming t-SNE converges to a local minimum of $D(P\Vert Q)$ so that it terminates close to one of these local minima, these practical modifications of step-size, early exaggeration, and momentum do not affect the objective function $D(P\Vert Q)$, and only impact which local minimum is achieved and how quickly it converges to that local minimum.
In particular, during the early exaggeration phase with $\alpha>1$, the update steps no longer necessarily correspond to the gradient of some function, but because this is typically only used for the initial steps, e.g., the first 100 out of 1000 steps, afterwards the algorithm is still minimizing $D(P\Vert Q)$ via gradient descent.

\section{Global minimizer for points on a sphere}\label{section:sphere}
Recall the total variation distance $\dTV(P,Q)$ between two discrete distributions $P$ and $Q$ on some set $A$ is given by
\[ \dTV(P,Q) = \max_{B\subseteq A}\abs{P(B)-Q(B)} = \frac{1}{2}\sum_{x\in A}\abs{P(x)-Q(x)}. \]
It satisfies Pinsker's well-known inequality (see, e.g., \cite{CoverThomas})
\[ 2\dTV(P,Q)^2 \leq D(P\Vert Q). \]
The following lemma shows that for t-SNE output $Y$ where $Q$ is close to the uniform distribution $U$ in total variation distance, which is implied by $Y$ having small t-SNE objective value via Pinsker's inequality, we have almost all points of $Y$ are contained in a small ball. 
\begin{lemma}\label{lemma:small_KL}
    Consider an asymptotic regime as $n,d\to\infty$, where $Y=\set{y_1,\dots,y_n}\subset\R^2$ yields distribution $Q$ on $\binom{[n]}{2}$ and $U$ the uniform distribution on $\binom{[n]}{2}$.
    For function $r(n)=o(1)$, if $\dTV(U,Q)=o(r^3)$, then for all sufficiently large $n$, there exists an open ball of radius $r$ such that at least $1-r$ fraction of $Y$ lies in this ball.
\end{lemma}
\begin{proof}
    Let $a$ be the $(1-r)$-th quantile of the $\set{y_i-y_j}$ values over $ij\in\binom{[n]}{2}$.
    If $a<r$ we are done; by an averaging argument, there must exist a point $y_i$ such that $1-r$ fraction of the other points $y_j$ are within $a$ distance of $y_i$, and take our ball to be the radius-$r$ ball centered at $y_i$.
    
    Otherwise, we claim there exists a constant $\eps>0$ such that at most $\frac{23+\sqrt{17}}{32}+o(1)$ fraction of the pairs $ij\in\binom{[n]}{2}$ satisfy $(1-\eps)a\leq\norm{y_i-y_j}\leq(1+\eps)a$.
    By diagonalization, we instead prove that for all constants $c>0$, at most $\frac{23+\sqrt{17}}{32}+c+o(1)$ fraction of the pairs satisfy this inequality.
    Suppose this weren't the case; then there must exist $y_i$ such that at least $\frac{23+\sqrt{17}}{32}+c+o(1)$ fraction of the other points are in the annulus $A$ centered at $y_i$ with outer radius $(1+\eps)a$ and inner radius $(1-\eps)a$.
    Note that $y_i$ and the other $\frac{9-\sqrt{17}}{32}-c-o(1)$ fraction of the other points are part of at most $\paren{\frac{18-2\sqrt{17}}{32}-2c+o(1)}\binom{n}{2}$ of the at least $\paren{\frac{23+\sqrt{17}}{32}+c+o(1)}\binom{n}{2}$ pairs $jk\in\binom{[n]}{2}$ satisfying $(1-\eps)a\leq\norm{y_j-y_k}\leq(1+\eps)a$.
    Thus, there are $\paren{\frac{5+3\sqrt{17}}{32}+3c-o(1)}\binom{n}{2}$ satisfying pairs among the points in $A$, meaning some $y_j\in A$ has at least $\frac{5+3\sqrt{17}}{32}+3c-o(1)$ fraction of the other $n-1$ points $y_k$ satisfying $(1-\eps)a\leq\norm{y_j-y_k}\leq(1+\eps)a$.
    Let $B$ be the corresponding annulus centered at $y_j$.
    Then at least $\frac{\sqrt{17}-1}{8}+4c-o(1)$ fraction of the $n$ points lie in $A\cap B$.
    For sufficiently small $\eps$, say $\eps\leq0.17$, it is clear that any two points in $A\cap B$ have distance not in $[(1-\eps)a,(1+\eps)a]$.
    Thus, at least $\paren{\frac{\sqrt{17}-1}{8}+4c}^2-o(1)$ fraction of the pairs $ij\in\binom{[n]}{2}$ do not satisfy $(1-\eps)a\leq\norm{y_i-y_j}\leq(1+\eps)a$, contradicting the assumption that at least $\frac{23+\sqrt{17}}{32}+c+o(1)$ fraction of the pairs satisfy this inequality, as
    \[ \paren{\frac{\sqrt{17}-1}{8}+4c}^2 > \frac{9-\sqrt{17}}{32} = 1 - \frac{23+\sqrt{17}}{32}. \]
    
    Thus, we either have at least $\frac{9-\sqrt{17}}{64}-o(1)$ fraction of the pairwise distances are less than $(1-\eps)a$, or at least $\frac{9-\sqrt{17}}{64}-o(1)$ fraction of the pairwise distances are more than $(1+\eps)a$.
    Let $n$ be sufficiently large so that this $\frac{9-\sqrt{17}}{64}-o(1)$ expression is greater than $r=o(1)$.
    In the latter case, at least $1-r$ fraction of the pairwise distances are at most $a$, while more than $r$ fraction of the pairwise distances are at least $(1+\eps)a$; this immediately yields a contradiction.
    In the former case, we know at least $r$ fraction of the pairwise distances are at least $a$ and at least $r$ fraction of the pairwise distances are less than $(1-\eps)a$.
    As
    \[ \frac{1+a^2}{1+(1-\eps)^2a^2}\geq\frac{1+r^2}{1+(1-\eps)^2r^2}=1+\Theta(r^2) \]
    and $\frac{1+x}{1-x}=1+\Theta(x)$ as $x\to0$, either at least $r$ fraction of the $q_{ij}$ are at most $\frac{1-\Theta(r^2)}{\binom{n}{2}}$ or at least $r$ fraction of the $q_{ij}$ are at least $\frac{1+\Theta(r^2)}{\binom{n}{2}}$.
    Either way, this ensures $\dTV(U,Q)=\Omega(r^3)$, yielding a contradiction for all sufficiently large $n$ as we assumed $\dTV(U,Q)=o(r^3)$.
\end{proof}
Pinsker's inequality immediately implies the following corollary.
\begin{corollary}\label{corollary:small_KL}
    Using the notation of \cref{lemma:small_KL}, if $D(U\Vert Q)=o(r^6)$ or $D(Q\Vert U)=o(r^6)$, then for all sufficiently large $n$, there exists an open ball of radius $r$ such that at least $1-r$ fraction of $Y$ lies in this ball.
    Thus, if $D(U\Vert Q)=o(1)$ or $D(Q\Vert U)=o(1)$, then there exists a ball of radius $o(1)$ such that $1-o(1)$ fraction of $Y$ lies in this ball.
\end{corollary}
This shows that if $P$ is uniform and $Y$ yields good objective function, namely $D(U\Vert Q)=o(1)$ so that $Q$ is close to uniform, then most of $Y$ must be contained in a small ball, i.e., t-SNE sends most points to essentially the same point.

\begin{remark}\label{remark:simplex}
    If $X$ is an equidistant set, i.e., all pairwise distances are equal, such as when $X$ is a regular simplex or an orthonormal frame, i.e., a collection of orthonormal vectors, then $P$ will be uniform regardless of the choice of $\sigma_i$, as each $P_i$ will be uniform.
    If all of the $y_i$ are at exactly the same point in $\R^s$, then $Q$ is also uniform and $D(P\Vert Q)=0$ is minimized.
    Thus any minimizer must have $D(P\Vert Q)=0$, meaning $Q$ is uniform and thus $Y$ is equidistant.
    No nontrivial equidistant set exists in $\R^s$ for $n>s+1$, which is a very weak condition for t-SNE; if $s=2$, then we simply require more than 3 datapoints.
    Thus, the only global minimizer of the t-SNE objective puts all datapoints at the same point if $X$ is an equidistant set with $n>s+1$.
\end{remark}
\begin{remark}\label{remark:perplexity_too_high}
    Similarly, if the $\sigma_i$ are set using a perplexity value $p$ and we have $p\geq n-1$, then as the maximum entropy of a probability distribution over a set of cardinality $m$ is $\log(m)$, t-SNE will attempt to set $\sigma_i\to\infty$ to have $P_i$ be uniform, and thus $P$ will also be uniform.
\end{remark}
\cref{remark:simplex,remark:perplexity_too_high} prove \cref{proposition:single_point}.
In both situations, $P$ will be uniform, so by \cref{corollary:small_KL}, any t-SNE visualization $Q$ with a good objective value, i.e., $D(P\Vert Q)=o(1)$, where $D(P\Vert Q)=0$ is achievable by putting all $y_i$ at a single point, will have most of $Y$ within a small ball, i.e., $1-o(1)$ fraction of $Y$ within a ball of radius $o(1)$.

The following lemma shows that if the bandwidths $\sigma_1,\dots,\sigma_n$ are all some fixed constant $\sigma$ and the datapoints in $X$ are drawn independently and uniformly at random from the unit sphere $\SS^{d-1}\subset\R^d$, then with high probability $P$ is close to uniform in the $\mc L^\infty$-norm.
\begin{lemma}\label{lemma:sphere_concentration}
    Let the $\sigma_i$ be some fixed constant $\sigma$.
    For constant $\delta\in(0,1/2)$, for some $n=e^{\Theta\paren{d^{2\delta}}}$, letting $x_1,\dots,x_n$ be uniformly random samples on $\SS^{d-1}$, with probability $1-e^{-\Omega\paren{d^{2\delta}}}$ we have $\abs{\binom{n}{2}p_{ij}-1} = O\paren{d^{-1/2+\delta}}$ for all $ij\in\binom{[n]}{2}$.
\end{lemma}
\begin{proof}
    For all distinct $i,j\in[n]$, as
    \[ \norm{x_i-x_j}^2 = \norm{x_i}^2 - 2\bra x_i,x_j\ket + \norm{x_j}^2 = 2 - 2\bra x_i,x_j\ket, \]
    where $\bra u,v\ket$ denotes the inner product in $\R^d$, we have
    \[ p_{j|i} = \frac{\exp(-\norm{x_j-x_i}^2/2\sigma^2) }{\displaystyle\sum_{k\neq i}\exp(-\abs{x_k-x_i}^2/2\sigma^2)} = \frac{\exp(\bra x_i,x_j\ket/\sigma^2)}{\sum_{k\neq i}\exp(\bra x_i,x_j\ket/\sigma^2)}. \]
    Levy's lemma on the concentration of measure on the sphere (see, e.g., \cite{Milman1986}) tells us that for $\eps=d^{-1/2+\delta}$, the $\eps$-extension of the equator covers $1-e^{-\Omega(d^{2\delta})}$ of the measure of the sphere.
    Thus, $\abs{\bra x_i,x_j\ket}> d^{-1/2+\delta}$ with probability $e^{-\Omega(d^{2\delta})}$.
    Taking $n=e^{cd^{2\delta}}$ for sufficiently small constant $c>0$, using a union bound, we have $\abs{\bra x_i,x_j\ket}\leq d^{-1/2+\delta}$ for all $ij\in\binom{[n]}{2}$ with probability $1-\binom{n}{2}e^{-\Omega\paren{d^{2\delta}}}=e^{-\Omega\paren{d^{2\delta}}}$.
    Under this event, $\abs{(n-1)p_{j|i}-1}= O\paren{d^{-1/2+\delta}}$ for all $i\neq j$.
    Thus, $\abs{\binom{n}{2}p_{ij}-1}= O\paren{d^{-1/2+\delta}}$ for all $ij$.
\end{proof}
\cref{theorem:sphere} immediately follows from \cref{lemma:sphere_concentration} and the following technical generalization of \cref{theorem:sphere}.
\begin{theorem}\label{theorem:small_ball_technical}
    Fix constant $\delta\in(0,1/2)$ and $r=\omega(d^{-1/6+\delta/3})$ with $r=\Omega\paren{\frac{1}{n}}$.
    Suppose $P$ satisfies
    \[ \abs{\binom{n}{2}p_{ij}-1}=O\paren{d^{-1/2+\delta}} \]
    for all $ij\in\binom{[n]}{2}$.
    Then for sufficiently large $n$, the global minimizer $Y$ of the t-SNE objective $D(P\Vert Q)$ has $1-r$ fraction of $Y$ within some ball of radius $r$ and has all of $Y$ contained within some ball of radius $O\paren{\sqrt{r}}$.
\end{theorem}
\begin{proof}
    We may assume $r=o(1)$, as replacing any $r=\Omega(1)$ with $r=o(1)$ while preserving $r=\omega\paren{d^{-1/6+\delta/3}}$ yields a stronger result.
    Let all the points $y_i$ be the same, so that $Q$ is the uniform distribution $U$.
    Then using the well-known relationship between KL-divergence and the chi-squared divergence, namely
    \[ D(P\Vert Q) = \sum_{ij\in\binom{[n]}{2}} P_{ij}\log\frac{P_{ij}}{Q_{ij}} \leq \sum_{ij\in\binom{[n]}{2}} P_{ij}\paren{\frac{P_{ij}}{Q_{ij}}-1} = \sum_{ij\in\binom{[n]}{2}} (P_{ij}-Q_{ij})\cdot\frac{P_{ij}-Q_{ij}}{Q_{ij}} = \chi^2(P\Vert U),\]
    we have
    \[ D(P\Vert U) \leq \chi^2(P\Vert U) = \sum_{ij}\frac{O\paren{d^{-1+2\delta}}}{\binom{n}{2}} = O\paren{d^{-1+2\delta}}. \]
    Thus, let $Y$ be the global minimizer with corresponding $Q$, so that
    \[ D(P\Vert Q)\leq D(P\Vert U)=O\paren{d^{-1+2\delta}}=o(r^6). \]
    Then by \cref{corollary:small_KL}, for all sufficiently large $n$ there exists an open ball $B$ of radius $r$ such that at least $1-r$ fraction of $Y$ lies in this ball.

    It now suffices to show that if such a ball $B$ exists, then $Y$ is contained within some ball of radius $O\paren{\sqrt{r}}$.
    As translating $Y$ does not change $Q$ and thus does not change the objective value, without loss of generality suppose $B$ is centered at the origin.
    Let $y_i$ be a point of maximum norm, i.e., distance from 0, in $Y$, and without loss of generality rotate $Y$ so that $y_i$ is at $(-x,0)$ for some $x\geq0$.
    Let $S\subset[n]\setminus\set{i}$ be the set of $j$ such that $y_j\in B$, and let $S^c=([n]\setminus\set{i})\setminus S$.

    Recall that less than $r$ fraction of $Y$ is not in $B$, i.e., has norm at least $r$, and all points in $Y$ have norm at most $x$.
    Two points in $B$ are at distance at most $2r$, one point in $B$ and another point in $Y$ are at distance at most $r+x$, and two points in $Y$ are at distance at most $2x$, so we bound the denominator of the $q_{k\ell}$'s by
    \begin{align*}
        &\sum_{k\ell\in\binom{[n]}{2}}\paren{1+\norm{y_k-y_\ell}^2}^{-1}
        \\ &\geq \binom{n}{2}\paren{1-O\paren{\frac{1}{n}}}\paren{(1-r)^2\paren{1+4r^2}^{-1}+2r(1-r)\paren{1+(r+x)^2}^{-1}+r^2\paren{1+4x^2}^{-1}}
        \\ &\geq \binom{n}{2}(1-O(r)),
    \end{align*}
    where the $1-O\paren{\frac{1}{n}}$ term accounts for the adjustments to the $(1-r)^2$, $2r(1-r)$, and $r^2$ terms needed as we are considering pairs of distinct elements, i.e., choosing without replacement.
    Note that the denominator is clearly at most $\binom{n}{2}$, so the denominator is $\binom{n}{2}(1-O(r))$.
    
    We now split into two cases of $x\leq1.5$ and $x\geq1.5$.
    If $x\leq1.5$, we consider the gradient from \eqref{equation:gradient}.
    If $x<r$, we are immediately done, so we assume $x\geq r$.
    In fact, for sake of contradiction, assume $x\geq C\sqrt{r}$ for some constant $C\geq10$ to be chosen later.
    As $Y$ yields a global minimum for $D(P\Vert Q)$, we know $\frac{\partial}{\partial y_i}D(P\Vert Q)=0$, or equivalently
    \begin{equation}\label{equation:small_ball_FOC_1}
        \sum_{j\neq i} p_{ij}q_{ij}(y_j-y_i) = \sum_{j\neq i}q_{ij}^2(y_j-y_i).
    \end{equation}
    By assumption, $p_{ij}=\paren{1+O\paren{d^{-1/2+\delta}}}\binom{n}{2}^{-1}=(1+o(r^3))\binom{n}{2}^{-1}$, so we can multiply \eqref{equation:small_ball_FOC_1} by $\binom{n}{2}$ times the denominator of the $q_{ij}$'s to obtain
    \begin{equation}\label{equation:small_ball_FOC_2}
        \sum_{j\neq i}\paren{1+\norm{y_j-y_i}^2}^{-1}(1+o(r^3))(y_j-y_i)
        = \sum_{j\neq i}(1+O(r))\paren{1+\norm{y_j-y_i}^2}^{-2}(y_j-y_i).
    \end{equation}
    Viewing $Y$ as a subset of $\C$ instead of $\R^2$ for notational simplicity when considering the $x$-coordinate of \eqref{equation:small_ball_FOC_1}, note that for $j\in S$,
    \[ \Re\paren{\paren{1+\norm{y_j-y_i}^2}^{-1}(y_j-y_i)} \geq (1+(x+r)^2)^{-1}(x-r) \geq \frac{x-r}{10} \geq \frac{x}{10}\paren{1-\frac{\sqrt{1.5}}{10}} \geq \frac{x}{15}. \]
    Meanwhile, for $j\in S^c$,
    \[ \Re\paren{\paren{1+\norm{y_j-y_i}^2}^{-1}(y_j-y_i)} \leq 2x. \]
    As less than $r$ fraction of $Y$ is not in $B$, this means the contribution of the $j\in S^c$ to the real part of the left hand side of \eqref{equation:small_ball_FOC_2} is at most $\frac{30r}{1-r}=O(r)$ times the contribution of the $j\in S$, so we can write the real part of the left hand side of \eqref{equation:small_ball_FOC_2} as
    \begin{equation}\label{equation:small_ball_FOC_LHS}
        (1+O(r))\sum_{j\in S}\paren{1+\norm{y_j-y_i}^2}^{-1}\Re(y_j-y_i).
    \end{equation}
    A similar argument implies we can write the real part of the right hand side of \eqref{equation:small_ball_FOC_2} as
    \begin{equation*}
        (1+O(r))\sum_{j\in S}\paren{1+\norm{y_j-y_i}^2}^{-2}\Re(y_j-y_i).
    \end{equation*}
    This implies there exists some constant $C_1>0$ such that for all sufficiently large $n$, if
    \begin{equation}\label{equation:small_ball_FOC_sufficient}
        \paren{1+\norm{y_j-y_i}^2}^{-1} \leq 1-C_1r
    \end{equation}
    for all $j\in S$, then the real part of the right hand side of \eqref{equation:small_ball_FOC_2} must be smaller than the real part of the left hand side of \eqref{equation:small_ball_FOC_2}, contradicting the equality.
    As $\norm{y_j-y_i}\geq x-r \geq x\paren{1-\frac{\sqrt{1.5}}{10}}$ for all $j\in S$, taking $x\geq C\sqrt{r}$ for sufficiently large $C\geq10$ will ensure \eqref{equation:small_ball_FOC_sufficient}.
    Thus, for sufficiently large $n$, if $x\leq 1.5$ then $x\leq C\sqrt{r}$ for some constant $C$, so $Y$ is contained within some ball of radius $O\paren{\sqrt{r}}$.

    Now we obtain a contradiction if $x\geq1.5$ and $n$ is sufficiently large, by arguing moving $y_i$ to 0 decreases $D(P\Vert Q)$.
    Note that after this move, the denominator of the $q_{k\ell}$'s is still $\binom{n}{2}(1-O(r))$, for a potentially different $O(r)$ term than before but still satisfying the same bound.
    Note
    \begin{equation*}
        D(P\Vert Q) = \sum_{k\ell\in\binom{[n]}{2}}p_{k\ell}^{-1}\log\frac{p_{k\ell}}{q_{k\ell}} = H(P) - \sum_{k\ell\in\binom{[n]}{2}}\paren{1+o\paren{r^3}}\binom{n}{2}^{-1}\log q_{k\ell}.
    \end{equation*}
    Moving $y_i$ to 0 will cause a beneficial decrease in the $ij$ terms for $j\in S$ by increasing $q_{ij}$, while it may cause an increase in the $ij$ terms for $j\in S^c$ and may cause an increase in all terms due to the change in the denominator of the $q_{k\ell}$'s.
    Let $d_1$ denote the old denominator and $d_2$ denote the new denominator.
    We wish to show the total decrease exceeds the total increase in magnitude.
    For $j\in S$, note that $q_{ij}$ is originally at most $\frac{\paren{1+(x-r)^2}^{-1}}{d_1}$ while its new value is at least $\frac{\paren{1+r^2}^{-1}}{d_2}$, so the decrease from the $ij$ term is at least
    \[ \paren{1+o\paren{r^3}}\binom{n}{2}^{-1}\log(1+(x-r)^2) - \log(1+r^2) - \log(1+O(r)) \geq \binom{n}{2}^{-1}\paren{\log\paren{1+x^2} - O(r)}. \]
    As at least $1-r$ fraction of the $Y$ are in $B$, the total decrease from all $j\in S$ is at least $\frac{2}{n}\paren{\log\paren{1+x^2}-O(r)}$.
    Similarly, for $j\in S^c$, note that $q_{ij}$ is originally at most $\frac{1}{d_1}$ and is now at least $\frac{\paren{1+x^2}^{-1}}{d_2}$, so the increase from the $ij$ term is at most $\binom{n}{2}^{-1}(\log\paren{1+x^2}-O(r))$, and as less than $r$ fraction of the $Y$ are not in $B$, the total increase from all $j\in S^c$ is at most $\frac{2r}{n}\paren{\log\paren{1+x^2}-O(r)}$.
    Thus, the total decrease from all $ij$ for $j\neq i$ is at least $\frac{2}{n}\paren{\log\paren{1+x^2}-O(r)}$.
    Because the change in denominators affects all $k\ell\in\binom{[n]}{2}$ and not just the $ij$, we will need an improved bound on the change from $d_1$ to $d_2$.
    Note that the total increase from the $k\ell$ terms where $i\not\in\set{k,\ell}$ is
    \[ \paren{1+o\paren{r^3}}\paren{1-O\paren{\frac{1}{n}}}\paren{\log\paren{d_2}-\log\paren{d_1}}. \]
    Note that for all $j\neq i$, the numerator of $q_{ij}$ starts at a value between $\paren{1+4x^2}^{-1}$ and 1, and ends at a value between $\paren{1+x^2}^{-1}$ and 1, so the change is at most $1-\paren{1+4x^2}^{-1}$.
    Thus,
    \[ \abs{d_2-d_1}\leq(n-1)\paren{1-\paren{1+4x^2}^{-1}}, \]
    while $d_1=(1-O(r))\binom{n}{2}$, so for sufficiently large $n$,
    \begin{align*}
        \log(d_2)-\log(d_1)
        &= \log\paren{1+\frac{d_2-d_1}{d_1}}
        \\ &\leq \log\paren{1+\frac{2}{n}(1+O(r))\paren{1-\paren{1+4x^2}^{-1}}}
        \\ &\leq \frac{2.1}{n}(1+O(r))\paren{1-\paren{1+4x^2}^{-1}}
        \\ &\leq \frac{2.2}{n}\paren{1-\paren{1+4x^2}^{-1}},
    \end{align*}
    where the second inequality assumes $n$ is sufficiently large so that the derivative of $\log(1+y)$ at $y=\frac{2}{n}(1+O(r))\paren{1-\paren{1+4x^2}^{-1}}$ is at most 1.05, and the last inequality uses the assumption that $r=o(1)$ and takes $n$ sufficiently large.
    For all $x\geq1.5$, we have $2\log(1+x^2)-2.2\paren{1-\paren{1+4x^2}^{-1}}$ is bounded below by a positive constant, and thus for sufficiently large $n$, we have $D(P\Vert Q)$ strictly decreases by moving $y_i$ to 0, contradicting optimality of $Y$.
    
    Combining these two cases, we see that $Y$ is contained within some ball of radius $O\paren{\sqrt{r}}$, which completes the proof.
\end{proof}
\begin{remark}\label{remark:robustness}
    It is straightforward to check that the approximate result, i.e., the $1-r$ fraction of $Y$ being contained in a ball of radius $r$, is robust to having up to $o(r^6)$ fraction of the pairs $ij\in\binom{[n]}{2}$ failing the $O\paren{d^{-1/2+\delta}}$ concentration bound, as each such $p_{ij}$ is still a probability, hence bounded between 0 and 1, so we can still obtain $D(P\Vert U)=o(r^6)$.
    Similarly, the gradient analysis, i.e., the argument for the case $x\leq1.5$, is also robust to some of the pairs failing the concentration bound, but because we only consider the $p_{ij}$ for some fixed $i$, we need each such ``row'' of $p_{ij}$ values to only have $O(r)$ fraction of its pairs fail the concentration bound in order for the argument to work.
    In particular, having $p_{ij}$ too large is not a problem, as this only increases the left hand side of our first order condition, \eqref{equation:small_ball_FOC_1}, which we were arguing was greater than the right hand side; meanwhile, if $O(r)$ fraction of the $p_{ij}$ for our fixed $i$ that maximizes $\norm{y_i}$ are too small, the smallest they can be is 0, and we can absorb this error in the $1+O(r)$ term in \eqref{equation:small_ball_FOC_LHS}.
    However, the macroscopic argument for the case $x\geq1.5$ is not robust to pairs $ij$ with $p_{ij}$ failing the concentration bound.
    This is because, unlike with the gradient analysis, having $p_{ij}$ too large can be an issue for the argument, where $p_{ij}=\Omega(1)$ is possible and would be much greater than the expected $\binom{n}{2}^{-1}$.
    There are some weaker and reasonable assumptions that can be imposed so that some $p_{ij}$ being too large is acceptable, such as the total probability mass of each row being not too much larger than $\frac{2}{n}$, the mass expected from $p_{ij}$ being uniform, but for simplicity we will not formally record the proof of this generalization as we did not find examples that needed this weakened assumption.
\end{remark}

We use \cref{corollary:small_KL} to prove \cref{proposition:doubled_frame}.
\begin{proof}[Proof of \cref{proposition:doubled_frame}]
    Let $A=\set{1,2,\dots,n}$ and $B=\set{n+1,n+2,\dots,2n}$.
    Then for distinct $i,j\in A$,
    \[ \binom{2n}{2}p_{ij} = (2n-1)\cdot\frac{e^{-\frac{2}{2\sigma^2}}}{(n-1)e^{-\frac{2}{2\sigma^2}}+ne^{-\frac{5}{2\sigma^2}}} \to \frac{2}{1+e^{-\frac{3}{2\sigma^2}}}. \]
    We denote this constant as $p_0^*$.
    Similarly, define $p_1^*=\frac{2}{1+e^{\frac{3}{2\sigma^2}}}=2-p_0^*<1<p_0^*$ and notice that for $i\in A$ and $j\in B$ or vice versa,
    \[ \binom{2n}{2}p_{ij} = \frac{2n-1}{2}\paren{\frac{e^{-\frac{5}{2\sigma^2}}}{(n-1)e^{-\frac{2}{2\sigma^2}}+ne^{-\frac{5}{2\sigma^2}}}+\frac{e^{-\frac{5}{2\sigma^2}}}{(n-1)e^{-\frac{8}{2\sigma^2}}+ne^{-\frac{5}{2\sigma^2}}}}\to\frac{p_0^*+p_1^*}{2}=1, \]
    and for distinct $i,j\in B$,
    \[ \binom{2n}{2}p_{ij} = (2n-1)\cdot\frac{e^{-\frac{8}{2\sigma^2}}}{(n-1)e^{-\frac{8}{2\sigma^2}}+ne^{-\frac{5}{2\sigma^2}}} \to p_1^*. \]
    Let $\mc A=\set{(A,A),(A,B),(B,B)}$.
    For $(I,J)\in\mc A$, let $P_{IJ}=\sum_{i\in I}\sum_{j\in J} p_{ij}$ and $Q_{IJ}=\sum_{i\in I}\sum_{j\in J}q_{ij}$, where if $I=J$ we omit the cases $i=j$ from the sums.
    These yield probability distributions $\mc P$ and $\mc Q$ on $\mc A$.
    Also let $P|(I,J)$ and $Q|(I,J)$ be the conditional distributions of $P$ and $Q$ given $(I,J)$, i.e., for $i\in I$ and $j\in J$, we have conditional probabilities $p_{ij}/P_{IJ}$ and $q_{ij}/Q_{IJ}$, respectively.
    Then we have
    \begin{align*}
        D(P\Vert Q)
        &= \sum_{(I,J)\in\mc A}\sum_{i\in I,j\in J}p_{ij}\log\frac{p_{ij}}{q_{ij}}
        \\ &= \sum_{(I,J)\in\mc A}P_{IJ}\sum_{i\in I,j\in J}\frac{p_{ij}}{P_{IJ}}\paren{\log\frac{p_{ij}/P_{IJ}}{q_{ij}/Q_{IJ}}+\log\frac{P_{IJ}}{Q_{IJ}}}
        \\ &= D(\mc P\Vert\mc Q) + \sum_{(I,J)\in\mc A}P_{IJ} D(P|(I,J)\Vert Q|(I,J))
        \\ &\geq \sum_{(I,J)\in\mc A}P_{IJ} D(P|(I,J)\Vert Q|(I,J)).
    \end{align*}
    Note that $P|(I,J)$ is the uniform distribution, for each $(I,J)\in\mc A$.
    For sake of contradiction, suppose there exists a subsequence of $n\in\N$ with $Y=Y_n$ such that $D(P\Vert Q)=o(1)$.
    As $P_{AA}\to\frac{p_0^*}{4}>0$ and $P_{BB}\to\frac{p_1^*}{4}>0$, under this subsequence, we have $D(U\Vert Q|(A,A))=o(1)$ and $D(U\Vert Q|(B,B))=o(1)$.
    Then by \cref{corollary:small_KL}, there exists a ball of radius $o(1)$ such that $1-o(1)$ fraction of $Y_A=\set{y_1,y_2,\dots,y_n}$ lies in this ball, and another ball of radius $o(1)$ such that $1-o(1)$ fraction of $Y_B=\set{y_{n+1},y_{n+2},\dots,y_{2n}}$ lies in this ball.
    Thus, $1-o(1)$ fraction of the pairs $ij\in\binom{A}{2}$, namely those pairs where both points are inside the small ball, have
    \[ q_{ij}\geq\frac{1-o(1)}{\displaystyle\sum_{k\ell\in\binom{[2n]}{2}} \paren{1+\norm{y_k-y_\ell}^2}^{-1}}, \]
    and similarly $1-o(1)$ fraction of the pairs $ij\in\binom{B}{2}$ satisfy the above inequality.
    Recall all $ij\in\binom{[2n]}{2}$ satisfy
    \[ q_{ij} \leq \frac{1}{\displaystyle\sum_{k\ell\in\binom{[2n]}{2}} \paren{1+\norm{y_k-y_\ell}^2}^{-1}}, \]
    so both $Q_{AA}$ and $Q_{BB}$ are
    \[ \frac{(1-o(1))\binom{n}{2}}{\displaystyle\sum_{k\ell\in\binom{[2n]}{2}} \paren{1+\norm{y_k-y_\ell}^2}^{-1}}\in\bracket{\frac{1-o(1)}{4},\frac{1+o(1)}{2}}. \]
    Thus $\abs{Q_{AA}-Q_{BB}}=o(1)$.
    However, $P_{AA}-P_{BB}\to\frac{p_0^*-p_1^*}{4}$, which is a positive constant, so by Pinsker's inequality,
    \[ D(P\Vert Q) \geq D(\mc P\Vert Q) \geq 2\dTV(\mc P,\mc Q)^2 \geq 2\paren{\frac{p_0^*-p_1^*-o(1)}{8}}^2 = \Omega(1), \]
    as desired.
\end{proof}

We conclude this section by proving \cref{proposition:volume}.
\begin{proof}[Proof of \cref{proposition:volume}]
    We first claim that with high probability, all distinct $x,y\in X$ have $0.2\leq\norm{x-y}\leq2$, as the distance-$r$ spherical cap for small $r$ has area like $\frac{1}{\sqrt{d}}r^{d-1}$, so a union bound gives $\frac{1}{\sqrt{d}}r^{d-1}\binom{2^d}{2}\to0$ for $r=0.2$.
    To justify this rigorously, note the probability measure $\mu_r$ of the distance-$r$ spherical cap for small $r$ is, for angle $\theta$ satisfying $r=2\sin\frac{\theta}{2}$,
    \[ \mu_r = \frac{\vol(\SS^{d-2})}{\vol(\SS^{d-1})}\int_0^\theta \sin^{d-2}\phi\di\phi \leq \frac{\Gamma\paren{\frac{d}{2}}}{\sqrt{\pi}\Gamma\paren{\frac{d-1}{2}}}\int_0^\theta\phi^{d-2}\di\phi \leq \sqrt{\frac{d}{2\pi}}\frac{\phi^{d-1}}{d-1} \leq \frac{1}{\sqrt{2\pi d}}\paren{\frac{r}{\sqrt{1-r^2/4}}}^{d-1}, \]
    where the last inequality follows from $\arcsin t \leq \frac{t}{\sqrt{1-t^2}}$ for $t\in[0,1]$.
    Thus, provided $\frac{r}{\sqrt{1-r^2/4}}\leq\frac{1}{2^2}$, we have the desired result, where $r=0.2$ satisfies this inequality.

    We work under the aformentioned event.
    Consider the length-$2B$ box centered at the origin, which $f(x)$ for all $x\in X$ must lie in.
    Partition this box into a grid of boxes of side length $\frac{(1+o(1))2Bg}{\sqrt{n}}$, where the $1+o(1)$ factor is to ensure the box is partitioned into an integer number of boxes.
    Thus, there are $(1-o(1))n/g^2$ boxes in total, so at most $\frac{1-o(1)}{g^2}=O\paren{g^{-2}}$ fraction of the points in $X$ are in their own box.
    For any other point $x$, picking another point $y$ in its box yields $\norm{f(x)-f(y)}\leq\frac{(2\sqrt{2}+o(1))Bg}{\sqrt{n}}$, which is at most $\frac{3Bg}{\sqrt{n}}$ eventually.
\end{proof}
\section{The split sphere example}\label{section:split_sphere}
As described in \cref{section:introduction}, let $\delta=0.49$ and consider a dataset $X$ obtained by drawing $n=e^{\Theta\paren{d^{2\delta}}}$ points uniformly from the unit sphere $\SS^{d-1}$, where $n$ comes from \cref{theorem:sphere}, or specifically \cref{lemma:sphere_concentration}.
Following the proof of \cref{lemma:sphere_concentration}, with high probability, specifically $1-e^{-\Omega\paren{d^{0.98}}}$, we have $\abs{\bra x_i,x_j\ket}\leq d^{-0.01}$ for all pairs $i,j\in[n]$.
Then, removing the points $x_i$ where the first coordinate $x_i^{(1)}$ has magnitude $\abs{x_i^{(1)}}<d^{-0.1}$, we continue to have $\abs{\bra x_i,x_j\ket}\leq d^{-0.01}$ for all remaining pairs of distinct points $x_i$ and $x_j$, and so the conclusion of \cref{lemma:sphere_concentration} and thus \cref{theorem:sphere} continue to hold for the split sphere dataset $X$.
Similarly, the conclusion of \cref{proposition:volume} also continues to hold for the split sphere, as the distance bound $\norm{x_i-x_j}\geq0.2$ holds for all $\binom{n}{2}$ pairs of distinct points, regardless of whether they were removed from $X$, so this bound continues to hold for all pairs in the split sphere $X$, and the proof continues identically.

We now justify the claim that with high probability, we still have $e^{\Theta\paren{d^{0.98}}}$ points remaining.
Note that a uniformly random point $x$ on $\SS^{d-1}$ can be obtained by taking i.i.d.\ standard normals $Z_1,\dots,Z_d$, and letting $x=\frac{1}{\sqrt{Z_1^2+\cdots+Z_d^2}}(Z_1,\dots,Z_d)$.
By a Chernoff bound, we have $\sqrt{Z_1^2+\cdots+Z_d^2}\leq 2\sqrt{d}$ with probability $1-e^{-\Omega(d)}$.
Thus, for sufficiently large $n$, we have by a union bound,
\begin{align*}
    \P\paren{\frac{\abs{Z_1}}{\sqrt{Z_1^2+\cdots+Z_d^2}}< d^{-0.1}}
    &\leq e^{-\Omega(d)} + \P\paren{\abs{Z_1}<2d^{0.4}} = e^{-\Omega(d)} + 1 - e^{-\Theta\paren{d^{0.8}}} = 1 - e^{-\Theta\paren{d^{0.8}}}.
\end{align*}
Thus, by a Chernoff bound, we keep at least $\frac{1}{2}e^{-\Theta\paren{d^{0.8}}}$ fraction of our $n$ points with probability $1-e^{-\Omega(n)}=1-e^{-\Omega\paren{d^{0.98}}}$, in which case we keep $\frac{1}{2}e^{\Theta\paren{d^{0.98}}}e^{-\Theta\paren{d^{0.8}}}=e^{\Theta\paren{d^{0.98}}}$ points in our split sphere dataset $X$.

Recall $x=\frac{1}{\sqrt{Z_1^2+\cdots+Z_d^2}}(Z_1,\dots,Z_d)$ yields a uniformly random point on the unit sphere $\SS^{d-1}$.
By the strong law of large numbers, we have $\frac{1}{\sqrt{d}}\sqrt{Z_1^2+\cdots+Z_d^2}\to1$ almost surely as $d\to\infty$.
This yields the heuristic that as $d\to\infty$, each coordinate of $x$ looks like a normal distribution with mean 0 and variance $\frac{1}{d}$.
Then, conditioning $x$ to be selected by $X$, i.e., having first coordinate of magnitude at least $d^{-0.1}$, scaling $X$ by $\sqrt{d}$, the first coordinate has magnitude at least $d^{0.4}$, and the other coordinates heuristically look like a standard normal.
Arora, Hu, and Kothari \cite[Corollary 3.2]{tSNE_visualization} show that a mixture of isotropic Gaussians, i.e., a mixture of Gaussians with identity covariance matrix $I_d$, can be fully visualized by t-SNE (see \cite[Definition 1.2]{tSNE_visualization} for the rigorous definition of a full visualization) with high probability as long as the means of the Gaussians are separated by $\widetilde\Omega\paren{d^{1/4}}$.
Our two spherical caps, when scaled to the isotropic Gaussian setting, are separated in the first coordinate by $2d^{0.4}$, which is certainly large enough for the $\widetilde\Omega\paren{d^{1/4}}$ condition.
Hence, one should expect that t-SNE on $X$ should yield a full visualization with respect to the two clusters of $X$, i.e., the output places the two spherical caps into two separate clusters.
\section{Numerical examples}\label{section:numerical_examples}
\subsection{High-dimensional spheres}\label{subsec:sphere_examples}
Inspired by \cref{theorem:sphere}, we look at numerical examples for when the datapoints are i.i.d., drawn uniformly from the unit sphere $\SS^{d-1}$.
\cref{fig:sphere_2} shows the two-dimensional t-SNE embedding, along with intermediate results, for the unit circle, i.e., the case $d=2$, with $n=1000$ points.
\begin{figure}[htbp]
    \centering
    \begin{subfigure}[b]{0.3\textwidth}
        \centering
        \includegraphics[width=\textwidth]{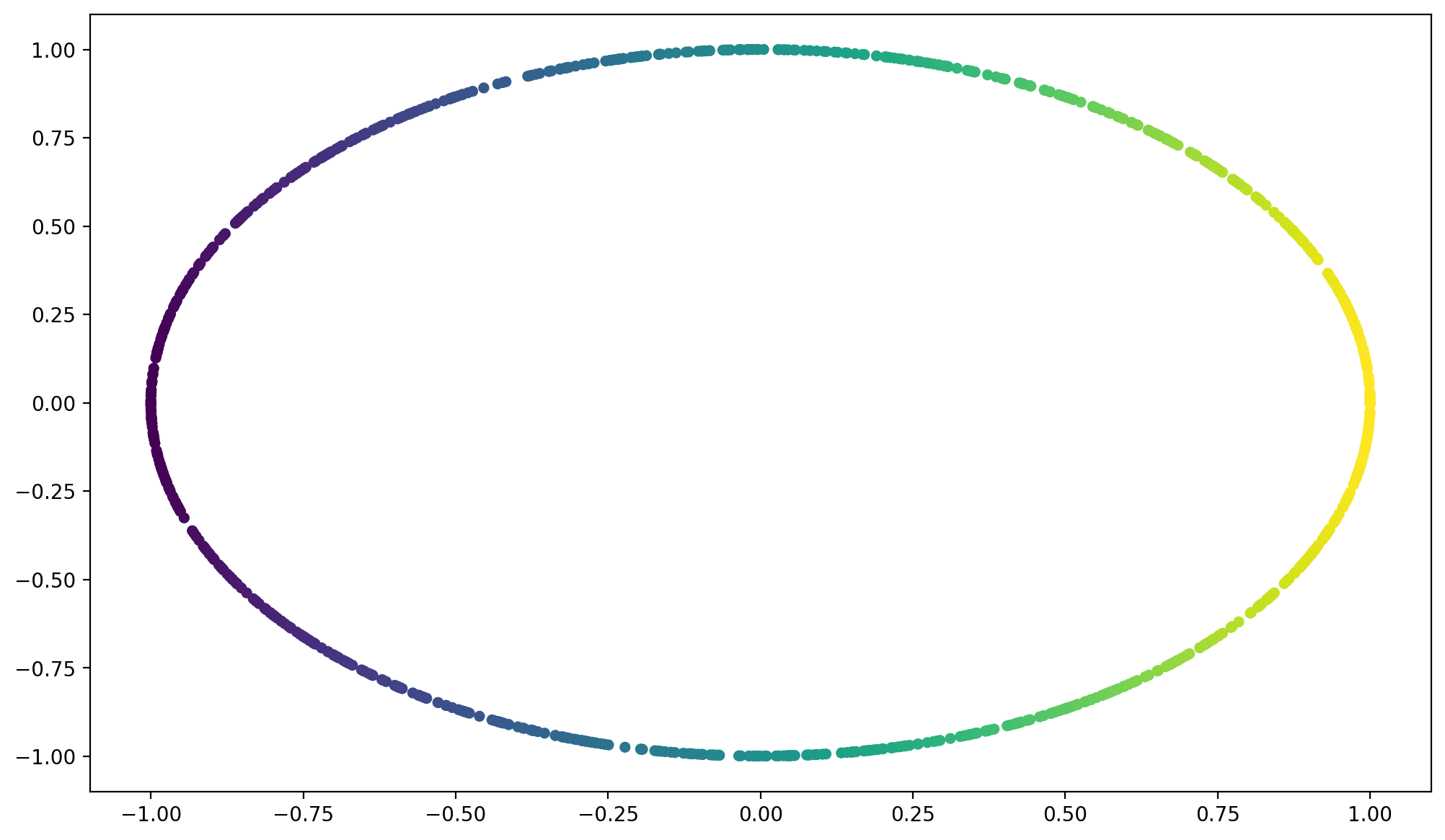}
        \caption{First two coordinates}
    \end{subfigure}
    \hfill
    \begin{subfigure}[b]{0.3\textwidth}
        \centering
        \includegraphics[width=\textwidth]{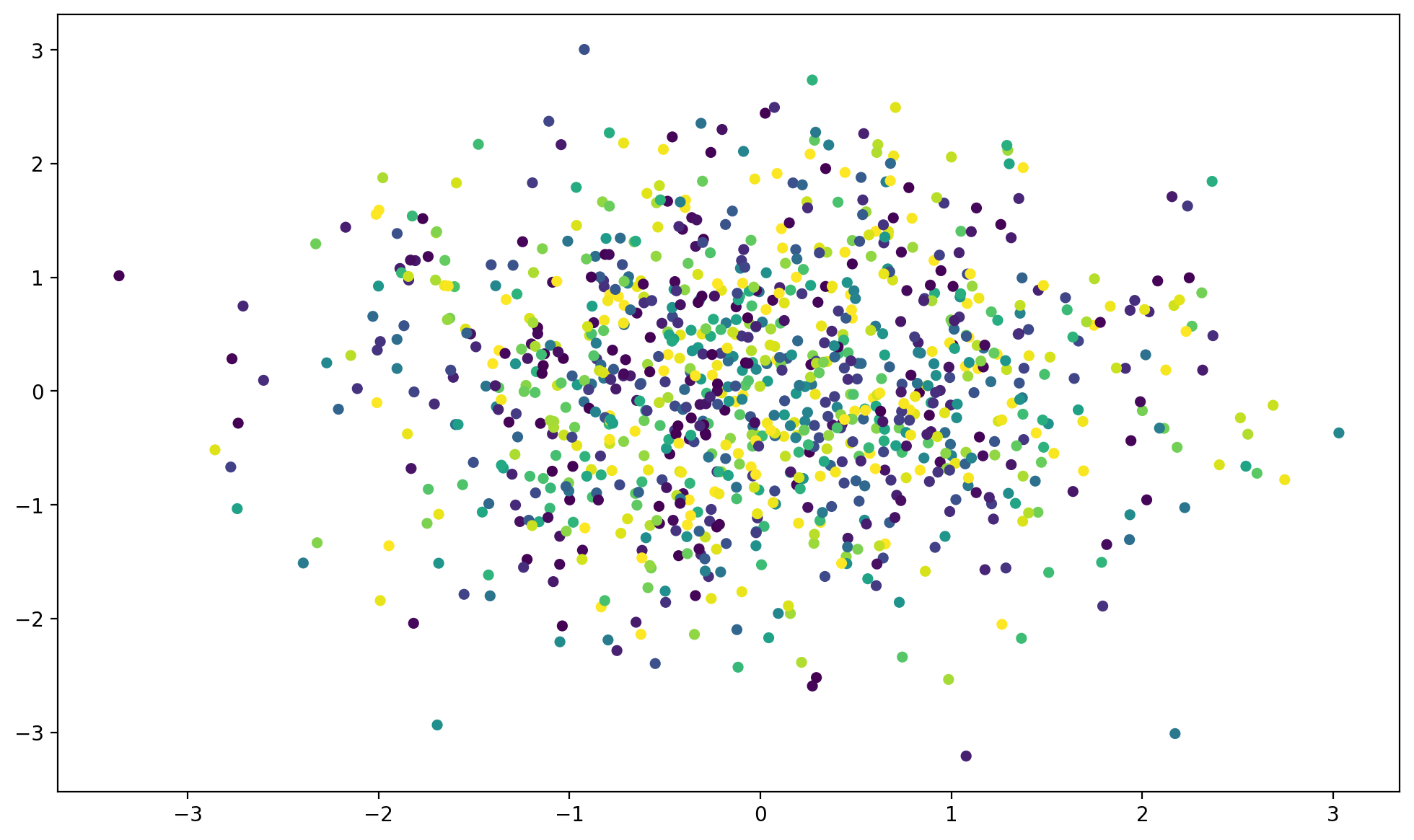}
        \caption{Initialization}
    \end{subfigure}
    \hfill
    \begin{subfigure}[b]{0.3\textwidth}
        \centering
        \includegraphics[width=\textwidth]{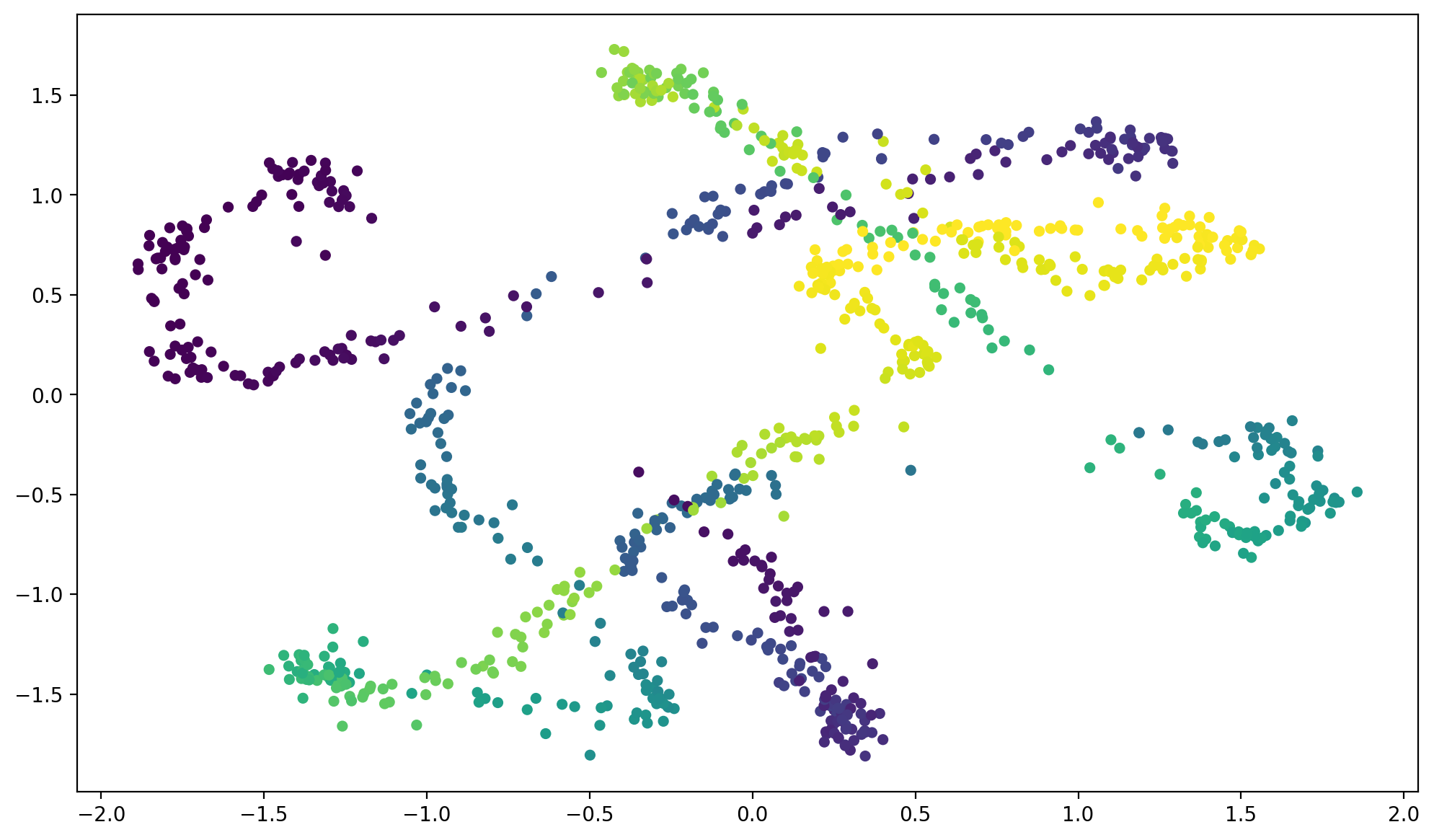}
        \caption{10 iterations}
    \end{subfigure}
    \begin{subfigure}[b]{0.3\textwidth}
        \centering
        \includegraphics[width=\textwidth]{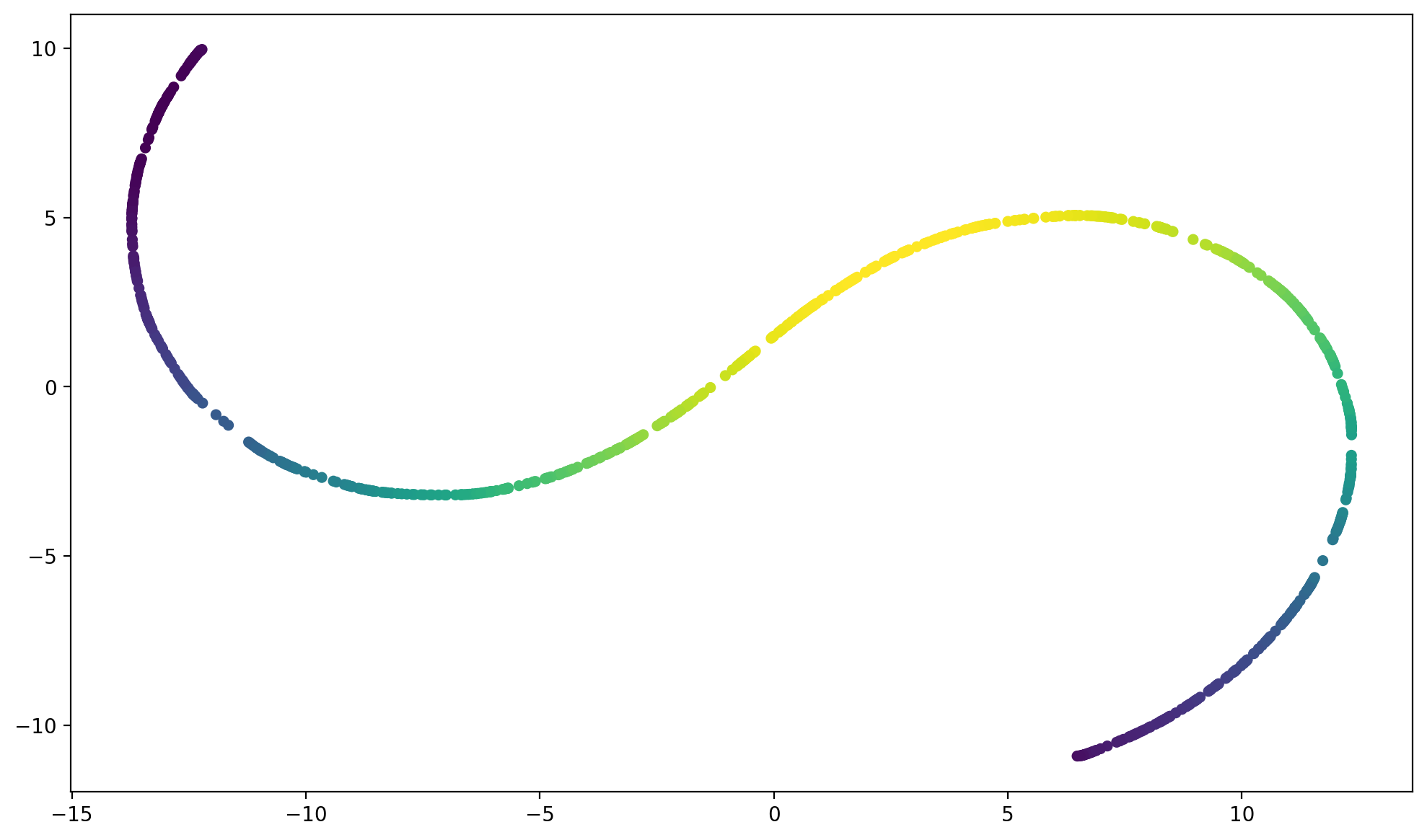}
        \caption{500 iterations}
    \end{subfigure}
    \hfill
    \begin{subfigure}[b]{0.3\textwidth}
        \centering
        \includegraphics[width=\textwidth]{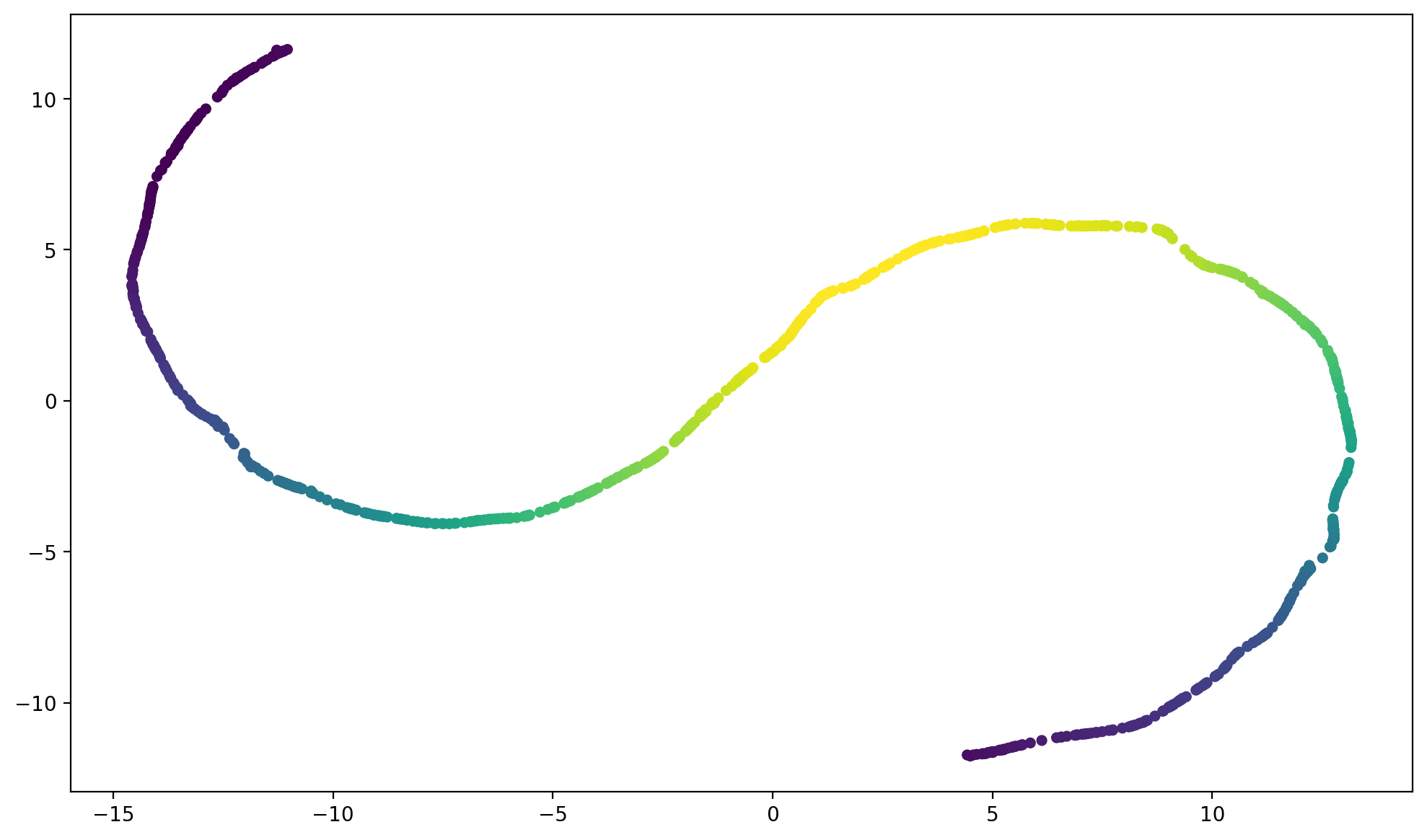}
        \caption{510 iterations}
    \end{subfigure}
    \hfill
    \begin{subfigure}[b]{0.3\textwidth}
        \centering
        \includegraphics[width=\textwidth]{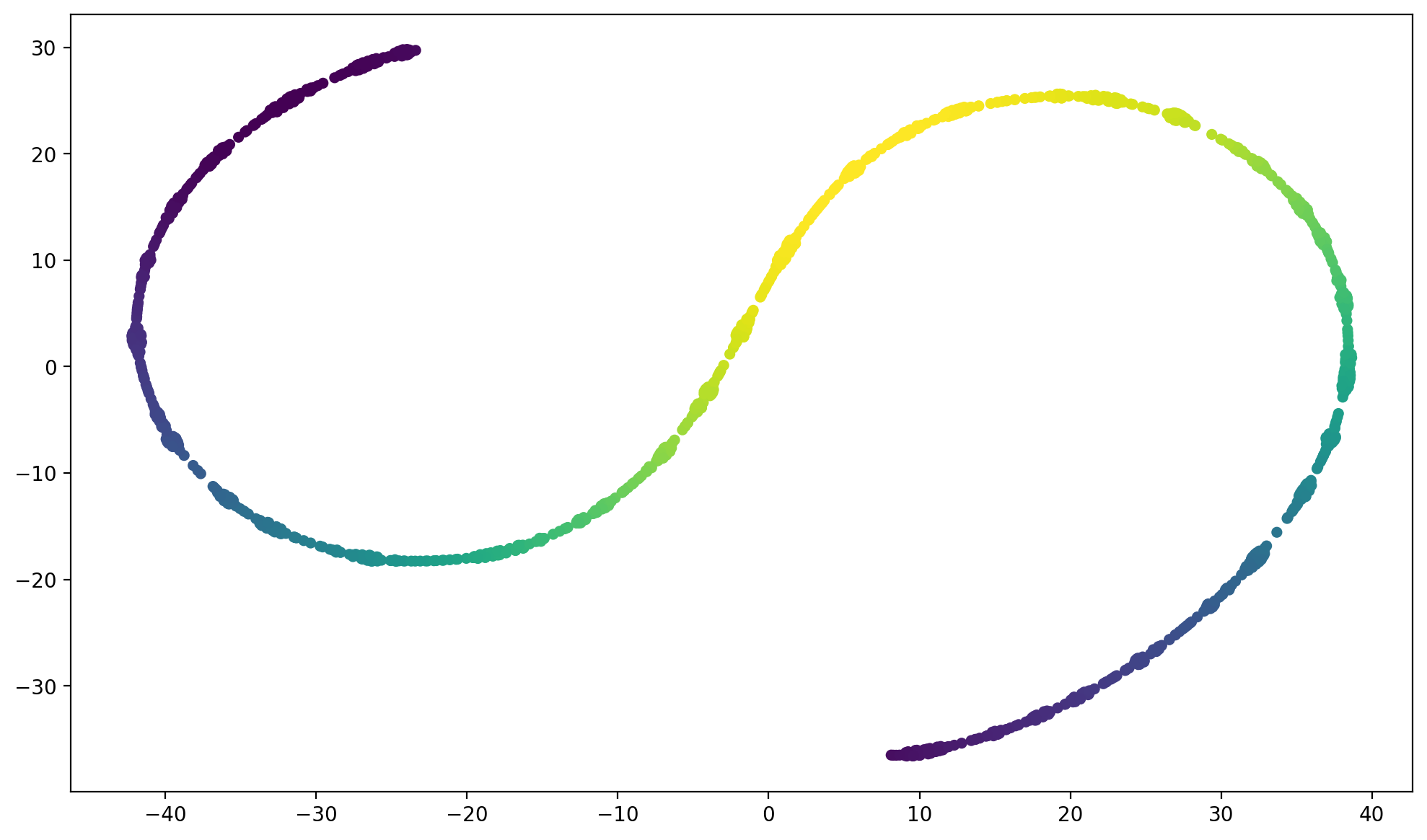}
        \caption{1000 iterations}
    \end{subfigure}

    \caption{t-SNE on points drawn from the unit sphere $\SS^1\subset\R^2$.}
    \label{fig:sphere_2}
\end{figure}
Here, we run t-SNE using the Python implementation of Laurens van der Maaten with default parameters, except changing early exaggeration to last for 500 iterations (instead of 100) out of the 1000 total to help demonstrate the distinction between t-SNE's behavior with and without early exaggeration.
We plot the dataset on its first two coordinates, which in the case of $d=2$ is all of its coordinates, show its initialization as a random collection of points in $\R^2$, then show the visualization after 10 iterations, 500 iterations (the end of early exaggeration), 510 iterations (10 steps after early exaggeration), and 1000 iterations, the end of the algorithm.
We color the datapoints by their first coordinate, to visually identify the points.
When $d=2$, we note that after 10 iterations, t-SNE is already well on its way to identifying the local structure, i.e., clusters points which were originally close together, and afterwards tightens these clusters and rearranges the points in a more visually clear manner.
We note that the final visualization may be disappointing because there are two ends that are not connected together; however, this visualization already captures most of the original structure very well, and tuning the parameters can easily enable t-SNE to output a closed loop in this case, but for sake of consistency as we increase dimension, we will continue using default parameters.

\begin{figure}[htbp]
    \centering
    \begin{subfigure}[b]{0.3\textwidth}
        \centering
        \includegraphics[width=\textwidth]{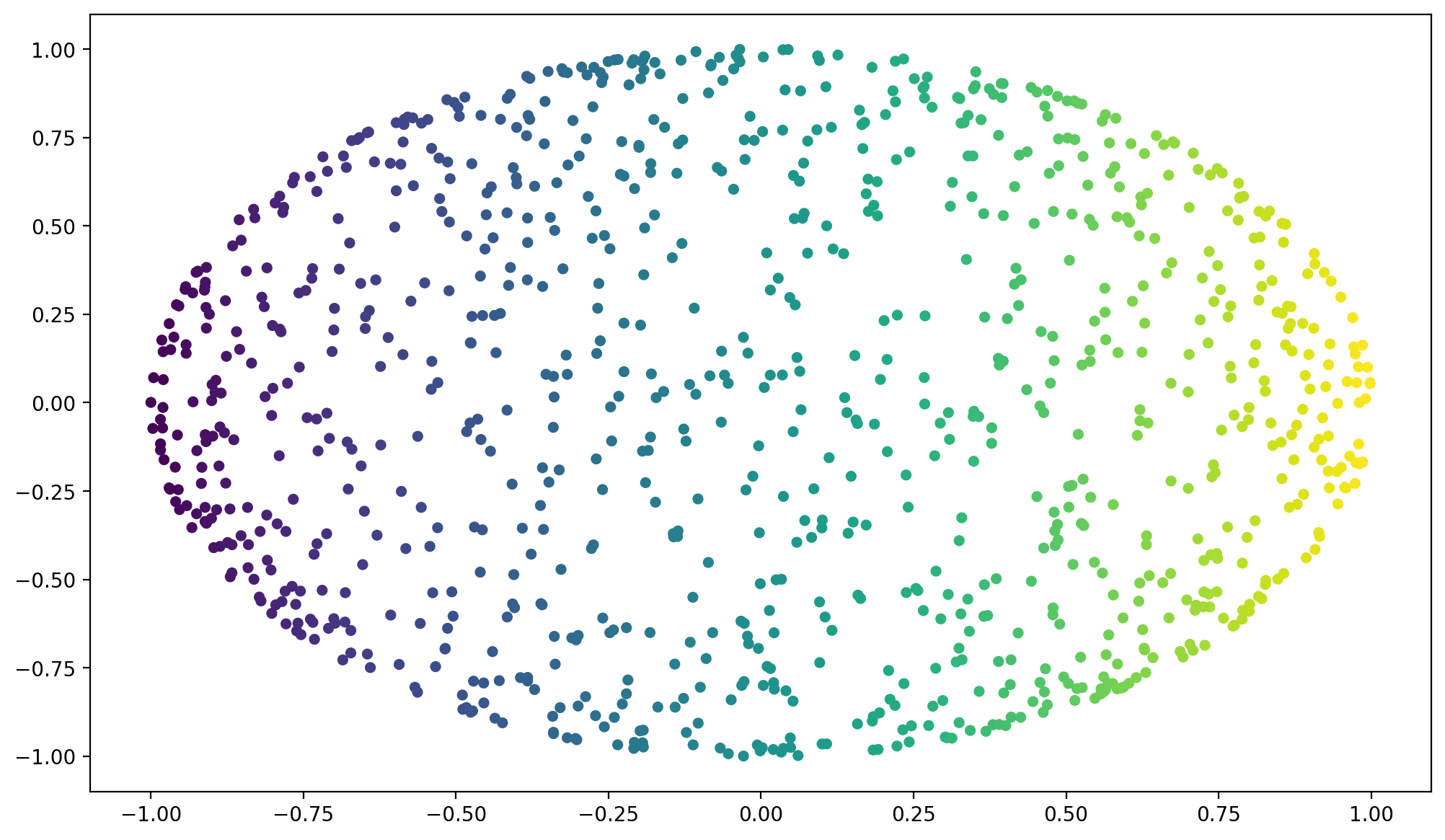}
        \caption{First two coordinates}
    \end{subfigure}
    \hfill
    \begin{subfigure}[b]{0.3\textwidth}
        \centering
        \includegraphics[width=\textwidth]{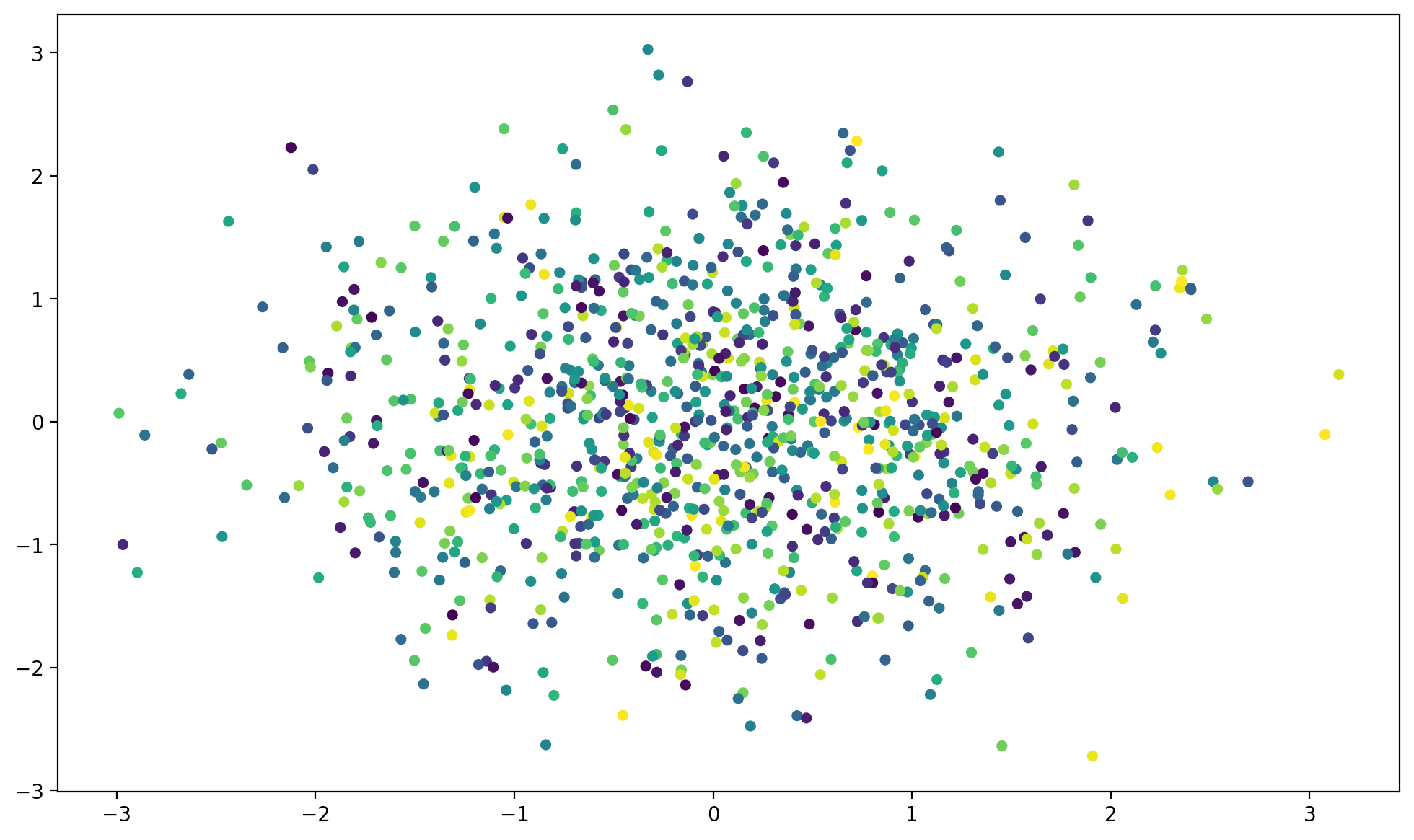}
        \caption{Initialization}
    \end{subfigure}
    \hfill
    \begin{subfigure}[b]{0.3\textwidth}
        \centering
        \includegraphics[width=\textwidth]{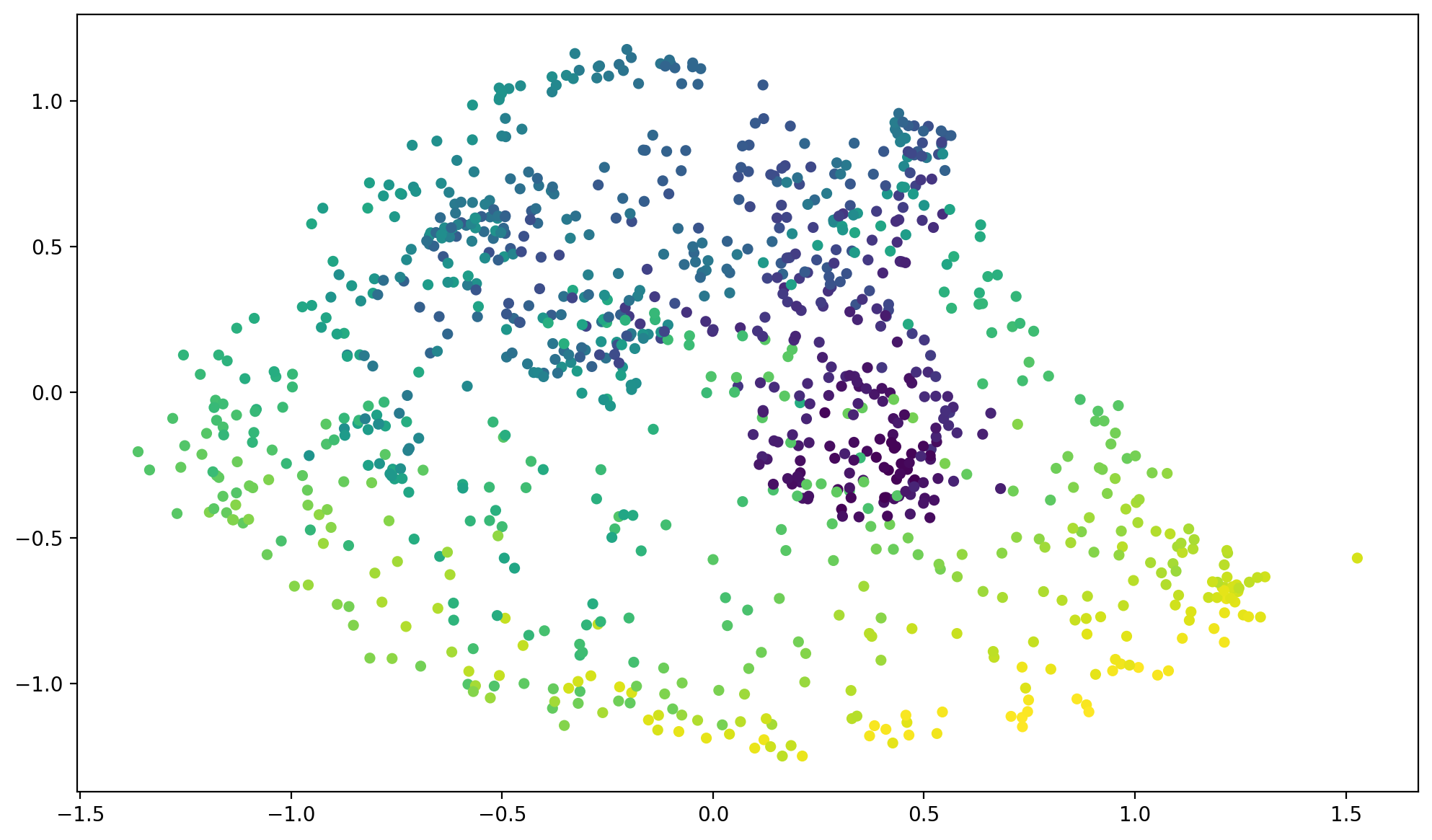}
        \caption{10 iterations}
    \end{subfigure}
    \begin{subfigure}[b]{0.3\textwidth}
        \centering
        \includegraphics[width=\textwidth]{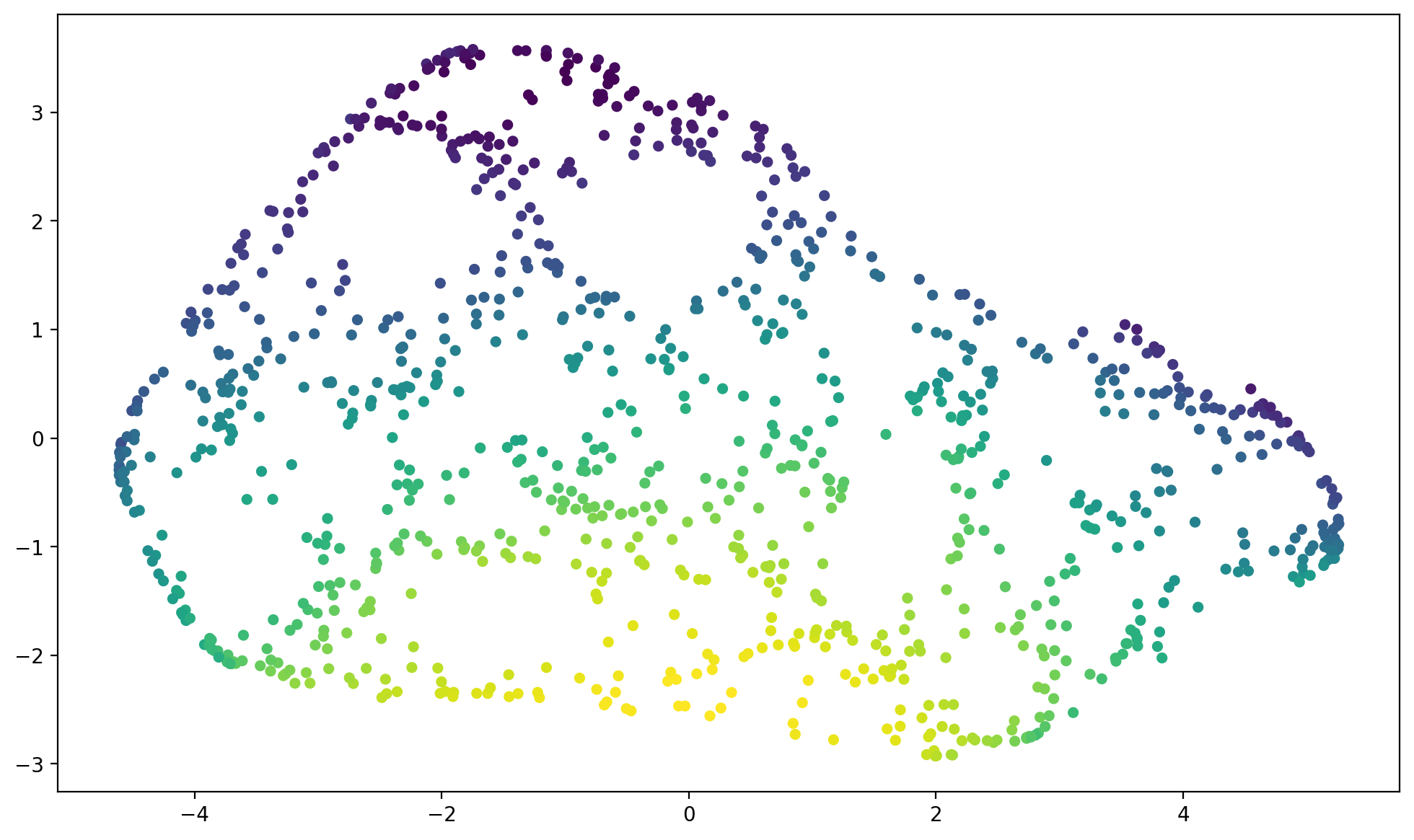}
        \caption{500 iterations}
    \end{subfigure}
    \hfill
    \begin{subfigure}[b]{0.3\textwidth}
        \centering
        \includegraphics[width=\textwidth]{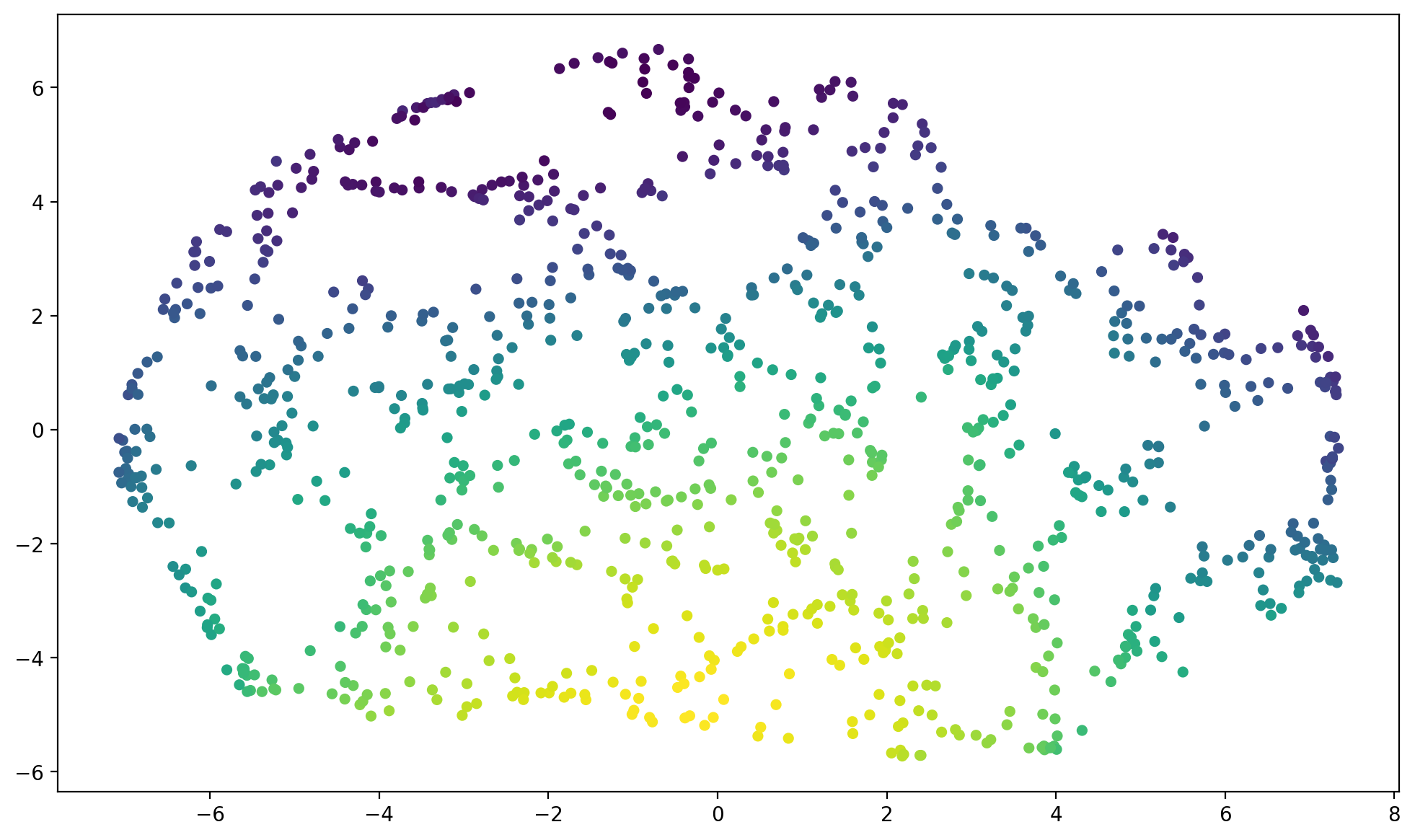}
        \caption{510 iterations}
    \end{subfigure}
    \hfill
    \begin{subfigure}[b]{0.3\textwidth}
        \centering
        \includegraphics[width=\textwidth]{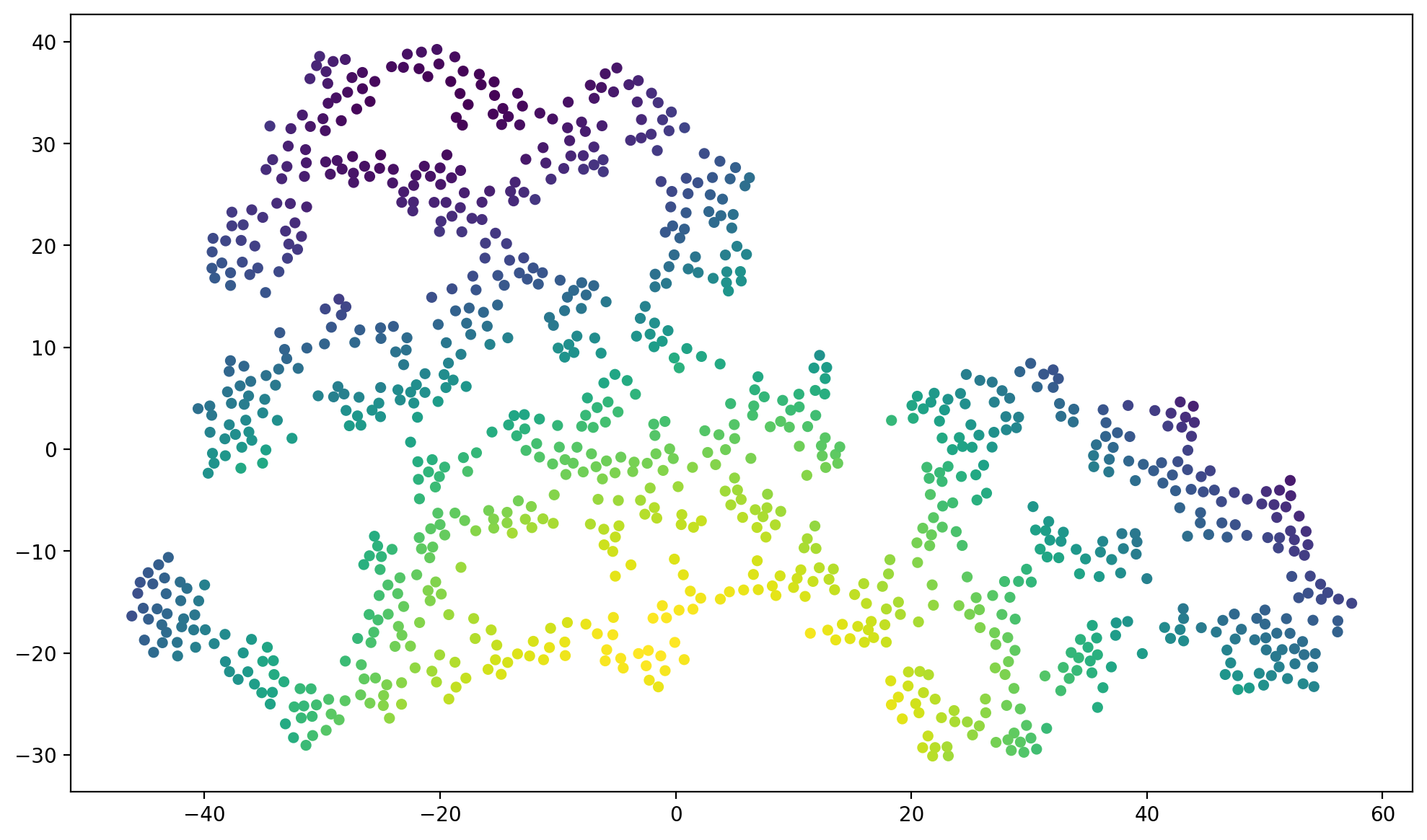}
        \caption{1000 iterations}
    \end{subfigure}

    \caption{t-SNE on points drawn from the unit sphere $\SS^2\subset\R^3$.}
    \label{fig:sphere_3}
\end{figure}
Next, we consider the unit sphere in three dimensions, with the analogous plots shown in \cref{fig:sphere_3}.
Note that now, despite the dataset being a spherical shell, projecting onto its first two dimensions yields a filled-in unit ball.
We can see the quality of the final visualization deteriorate compared to the case $d=2$, though the visualization is still able to identify local structure.

\begin{figure}[htbp]
    \centering
    \begin{subfigure}[b]{0.3\textwidth}
        \centering
        \includegraphics[width=\textwidth]{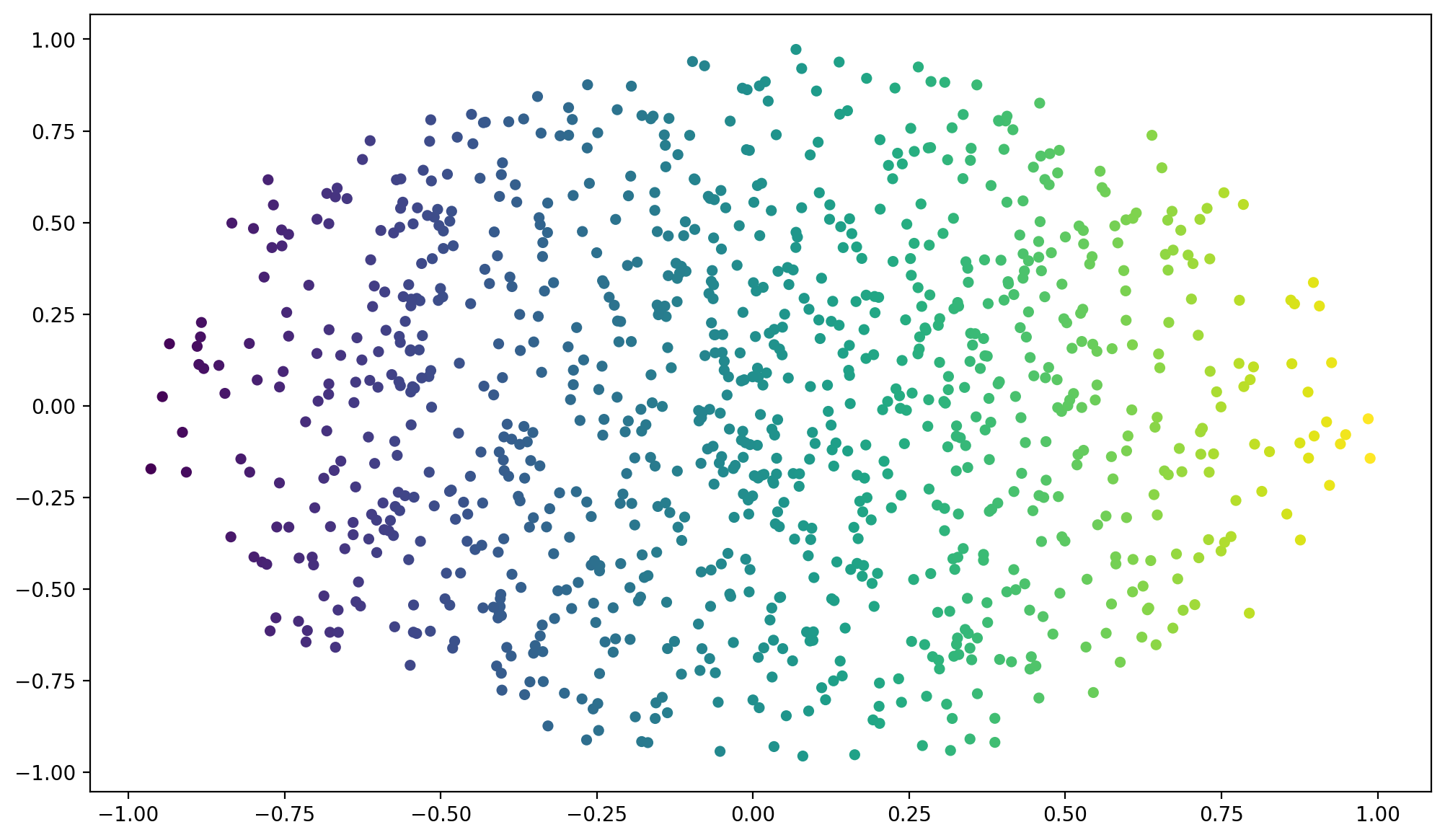}
        \caption{First two coordinates}
    \end{subfigure}
    \hfill
    \begin{subfigure}[b]{0.3\textwidth}
        \centering
        \includegraphics[width=\textwidth]{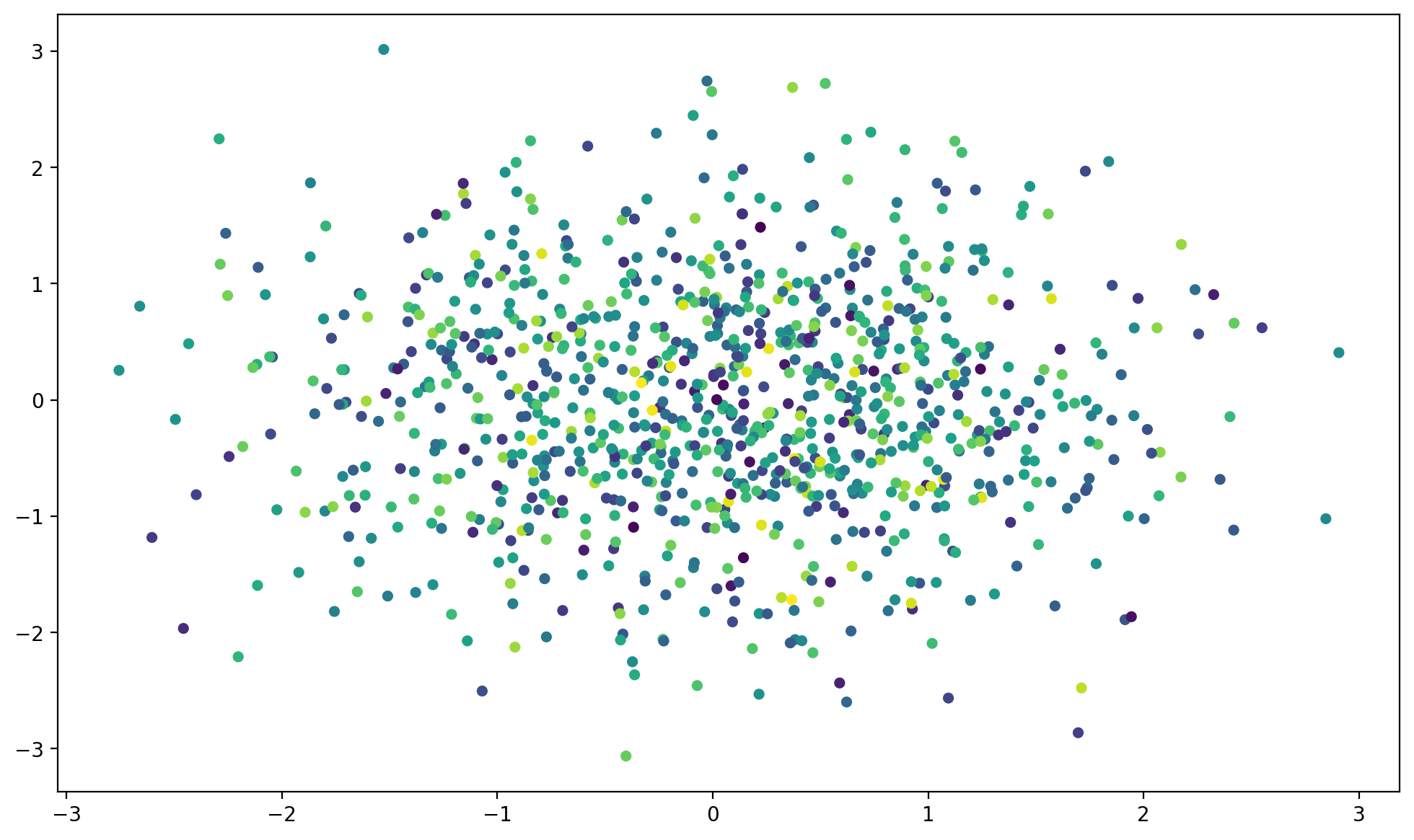}
        \caption{Initialization}
    \end{subfigure}
    \hfill
    \begin{subfigure}[b]{0.3\textwidth}
        \centering
        \includegraphics[width=\textwidth]{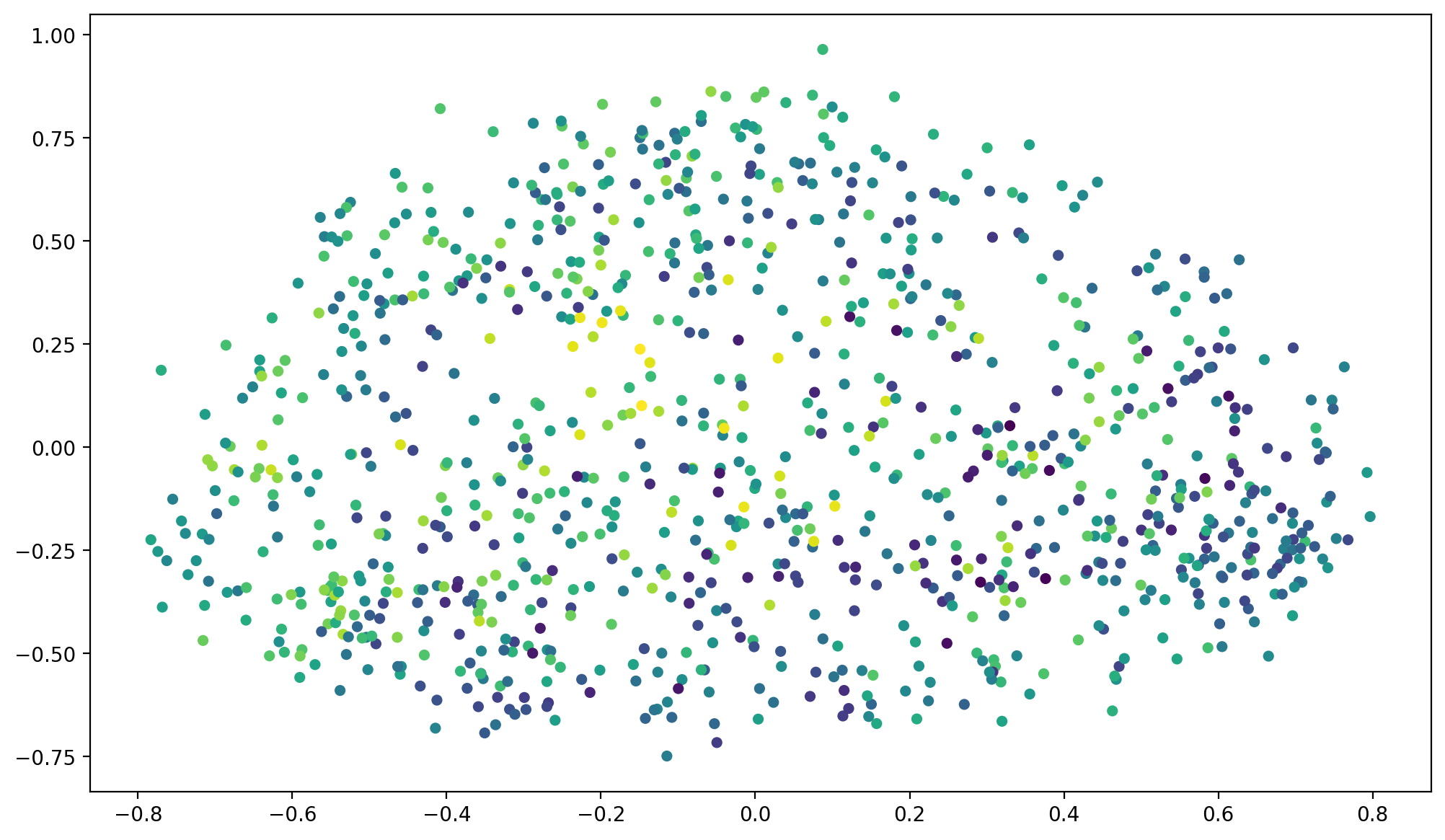}
        \caption{10 iterations}
    \end{subfigure}
    \begin{subfigure}[b]{0.3\textwidth}
        \centering
        \includegraphics[width=\textwidth]{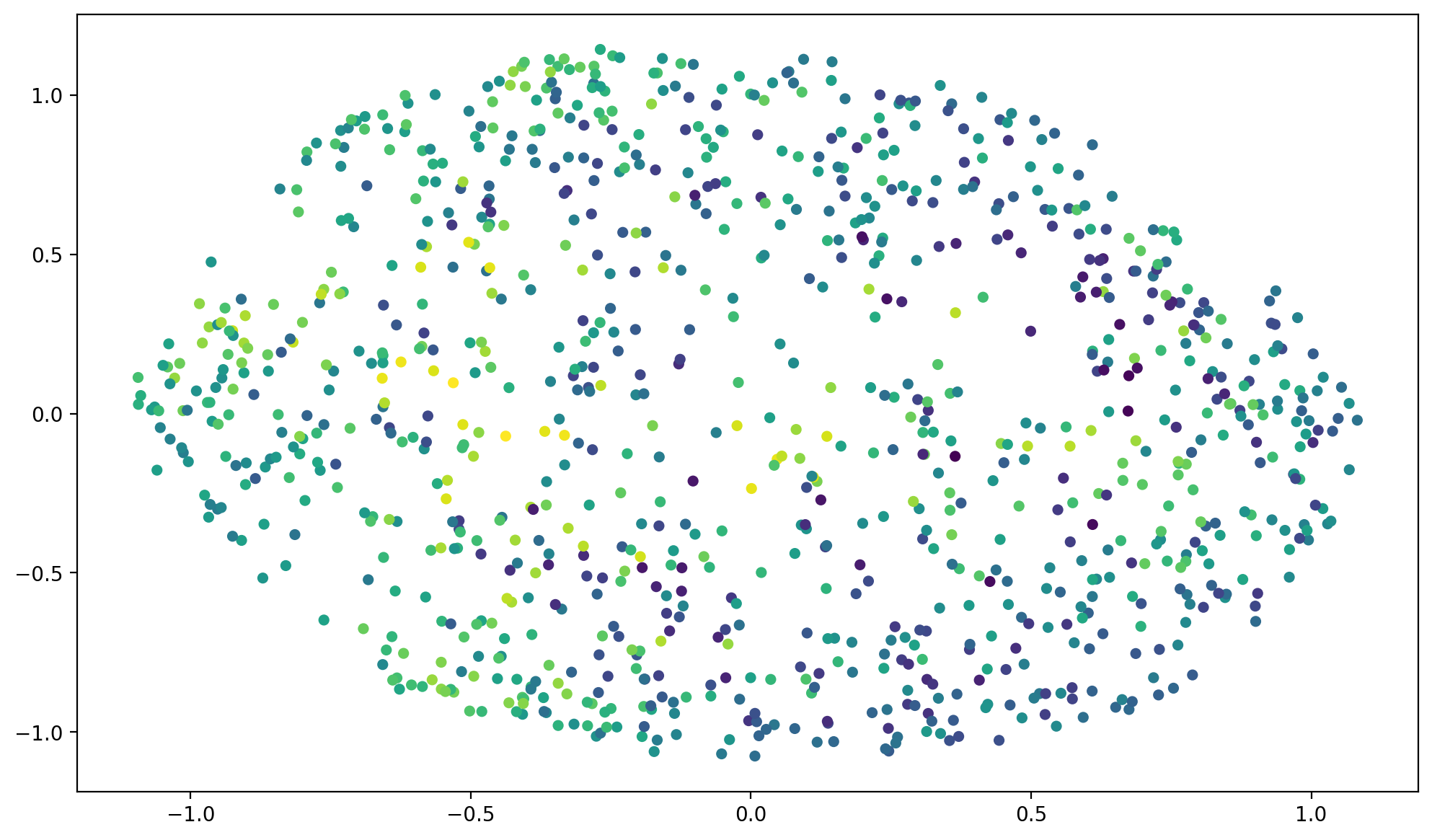}
        \caption{500 iterations}
    \end{subfigure}
    \hfill
    \begin{subfigure}[b]{0.3\textwidth}
        \centering
        \includegraphics[width=\textwidth]{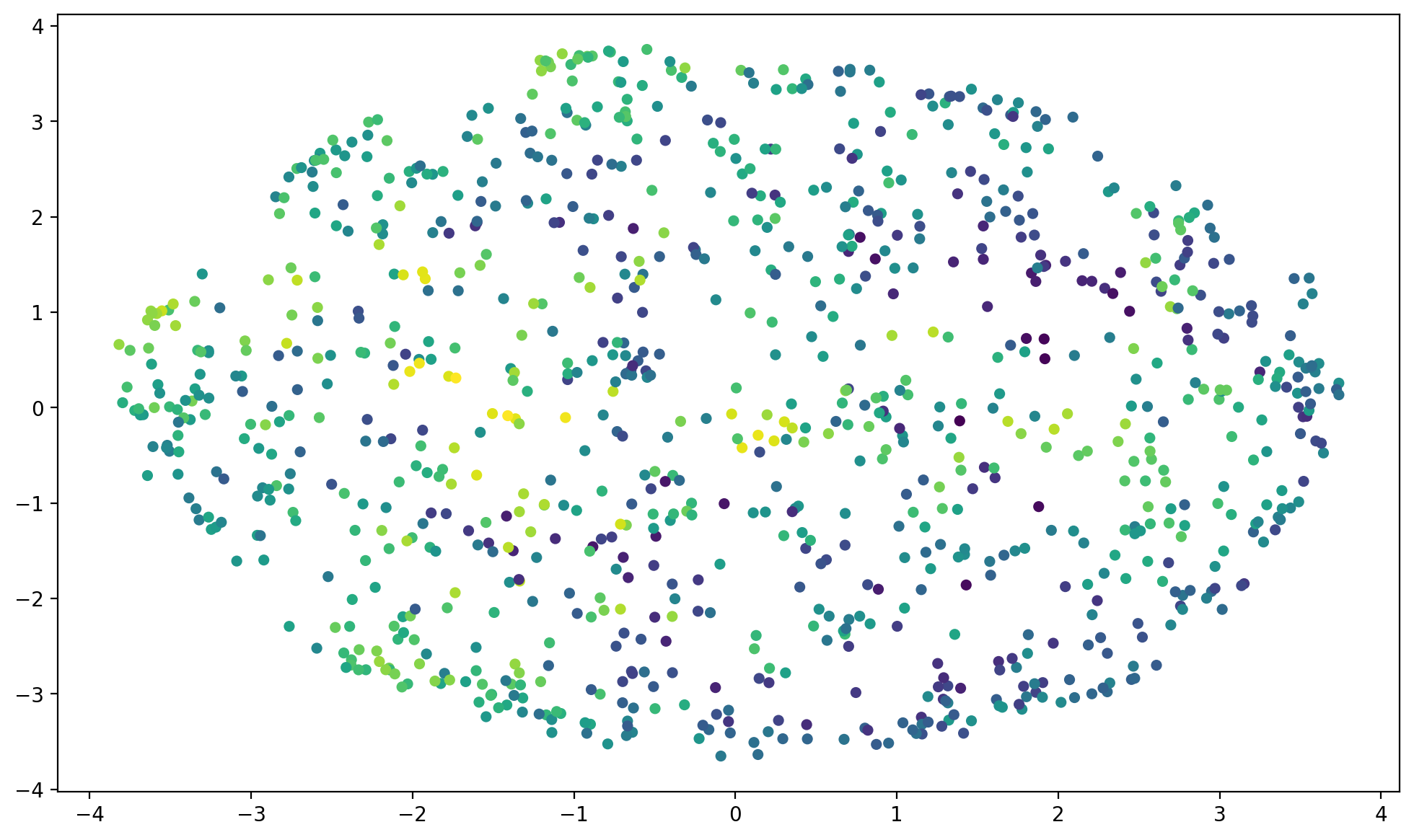}
        \caption{510 iterations}
    \end{subfigure}
    \hfill
    \begin{subfigure}[b]{0.3\textwidth}
        \centering
        \includegraphics[width=\textwidth]{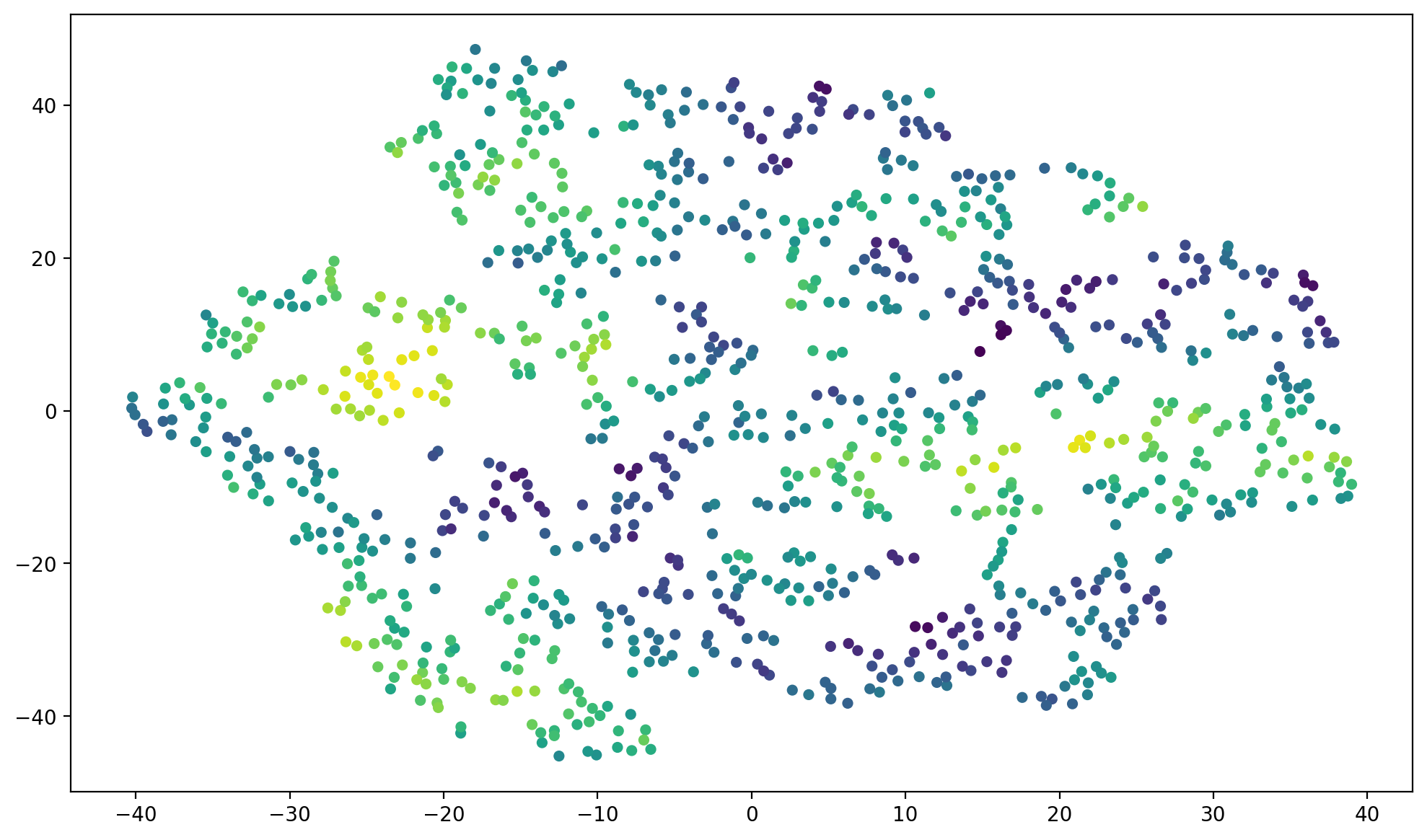}
        \caption{1000 iterations}
    \end{subfigure}

    \caption{t-SNE on points drawn from the unit sphere $\SS^4\subset\R^5$.}
    \label{fig:sphere_5}
\end{figure}
When we increase the dimension to $d=5$, plotted in \cref{fig:sphere_5}, the t-SNE visualization is now significantly worse, especially in the sense predicted by \cref{proposition:volume}: many points are visualized close to points that they are not close to in the original dataset, as can be seen by the mixing of the green and purple points.
There does appear to still be some preservation of local structure, however, e.g., the yellow points are still mostly surrounded by green points.

\begin{figure}[htbp]
    \centering
    \begin{subfigure}[b]{0.3\textwidth}
        \centering
        \includegraphics[width=\textwidth]{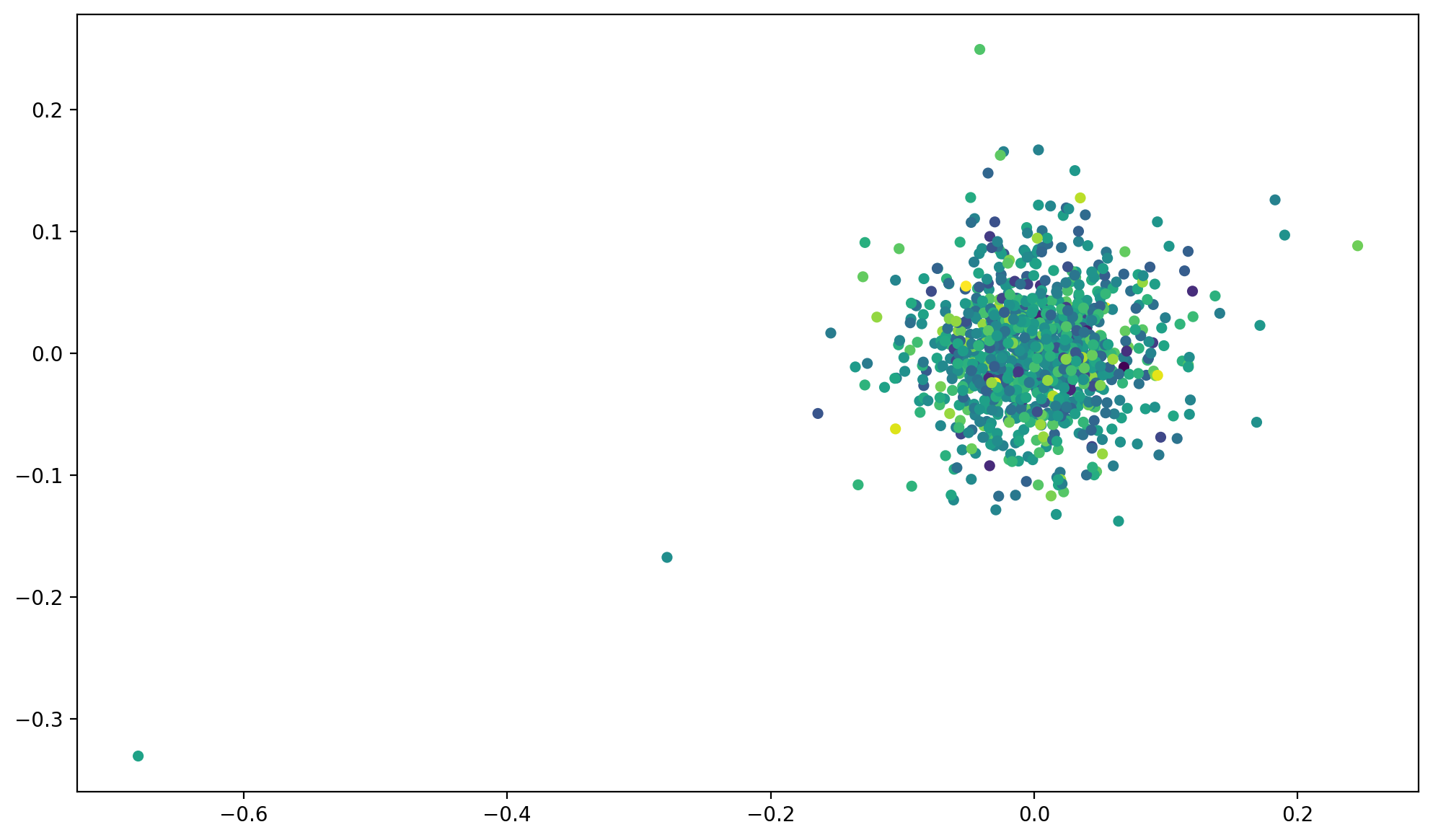}
        \caption{10 iterations}
    \end{subfigure}
    \hfill
    \begin{subfigure}[b]{0.3\textwidth}
        \centering
        \includegraphics[width=\textwidth]{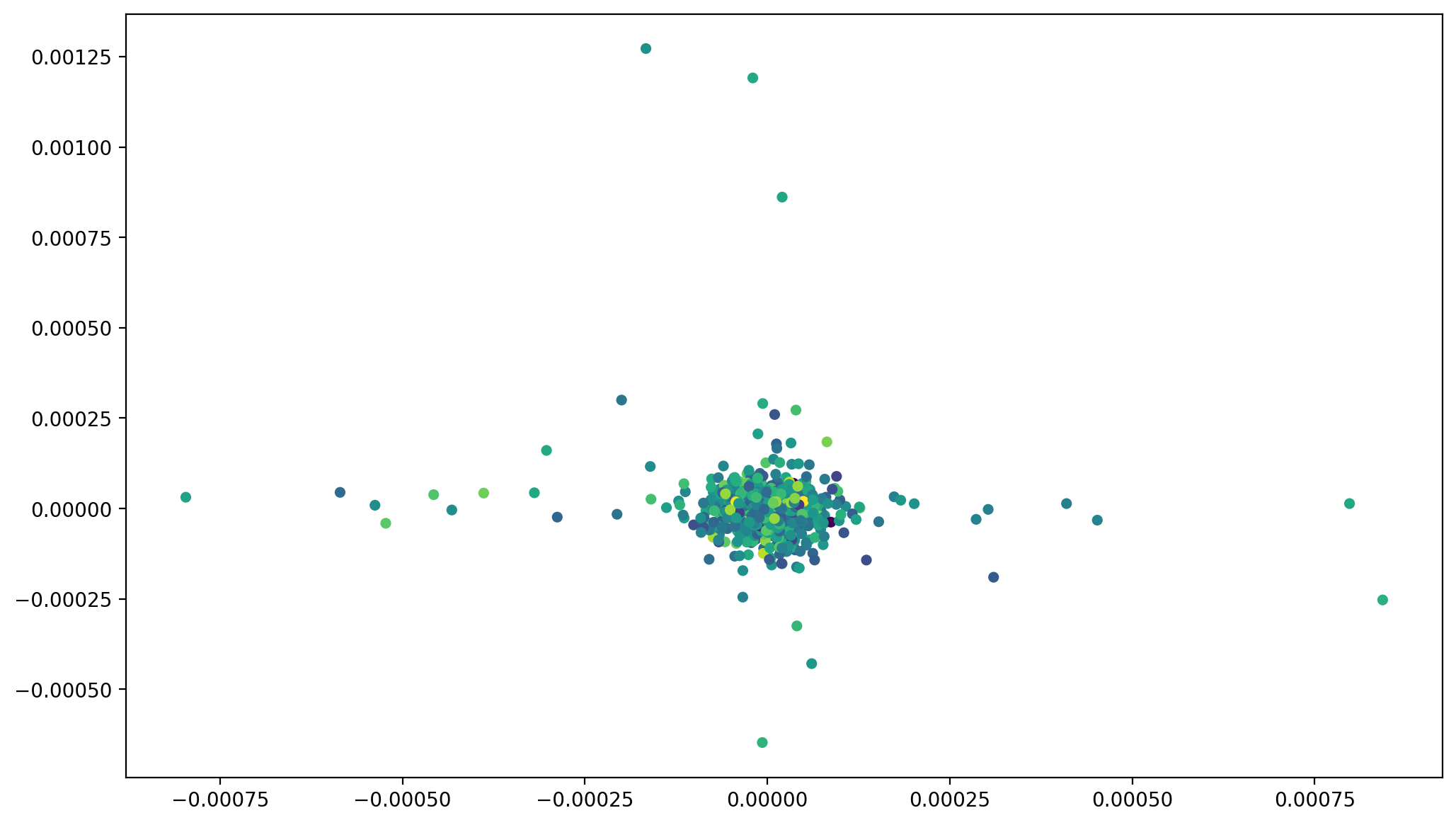}
        \caption{100 iterations}
    \end{subfigure}
    \hfill
    \begin{subfigure}[b]{0.3\textwidth}
        \centering
        \includegraphics[width=\textwidth]{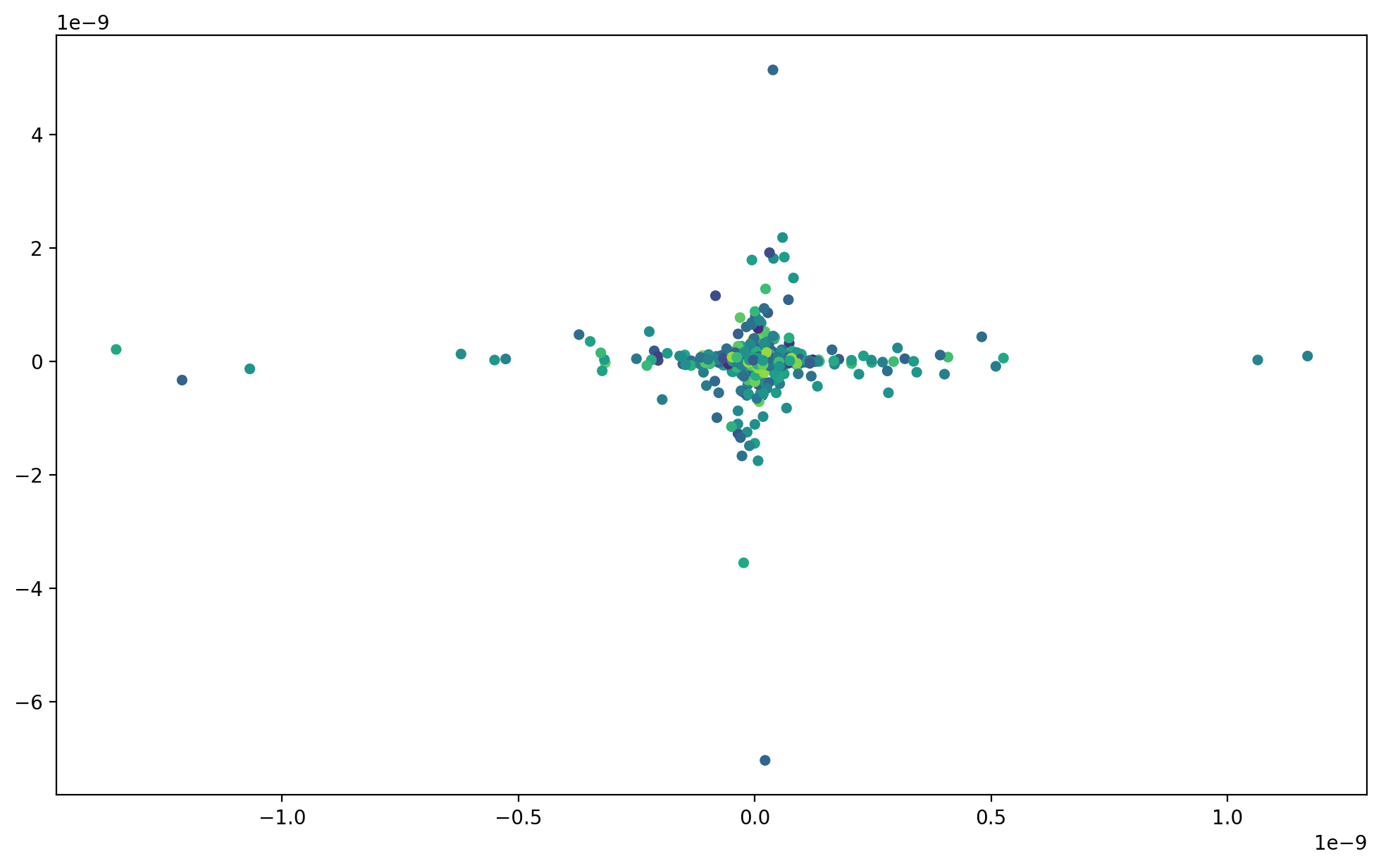}
        \caption{500 iterations}
    \end{subfigure}
    \begin{subfigure}[b]{0.3\textwidth}
        \centering
        \includegraphics[width=\textwidth]{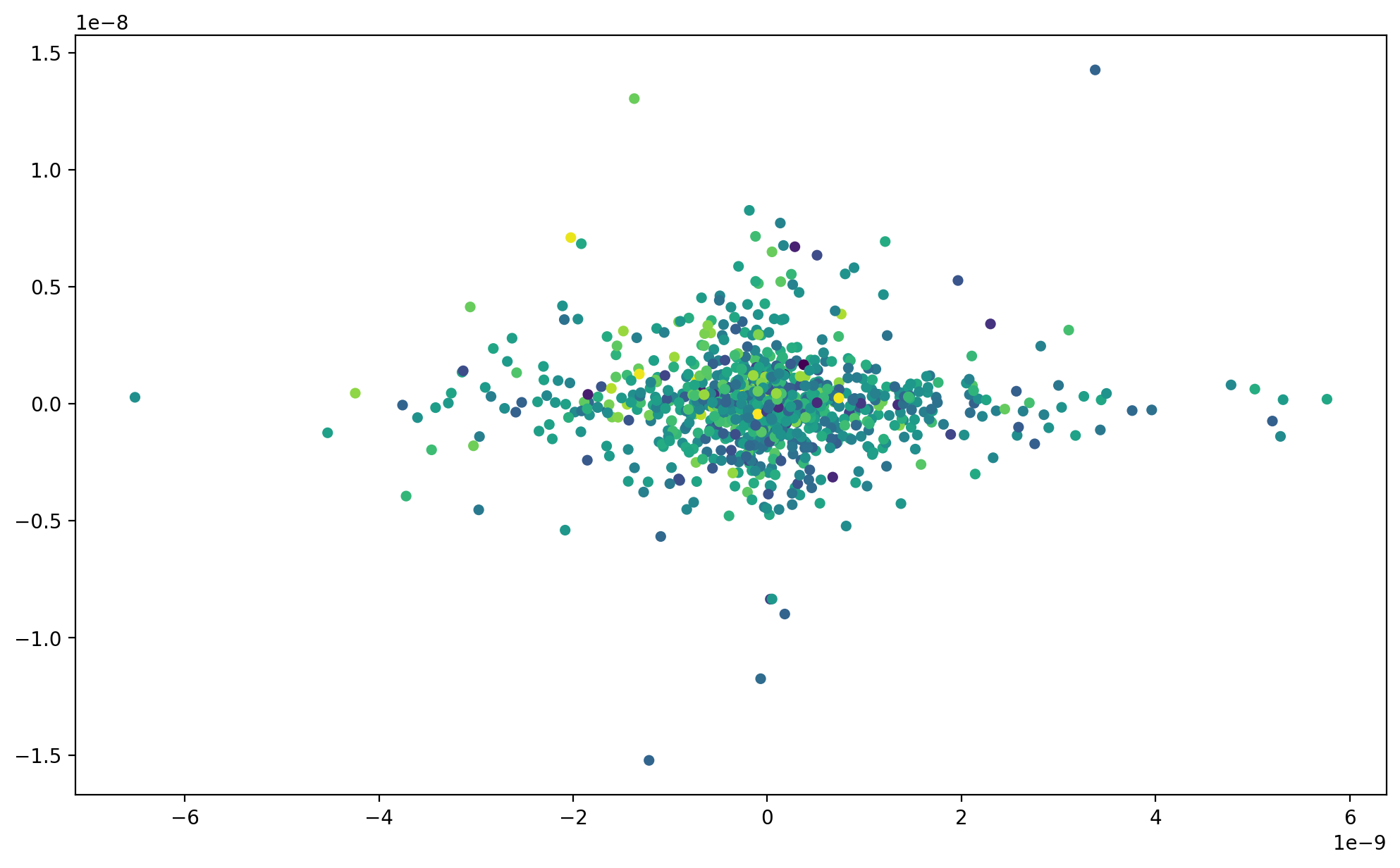}
        \caption{510 iterations}
    \end{subfigure}
    \hfill
    \begin{subfigure}[b]{0.3\textwidth}
        \centering
        \includegraphics[width=\textwidth]{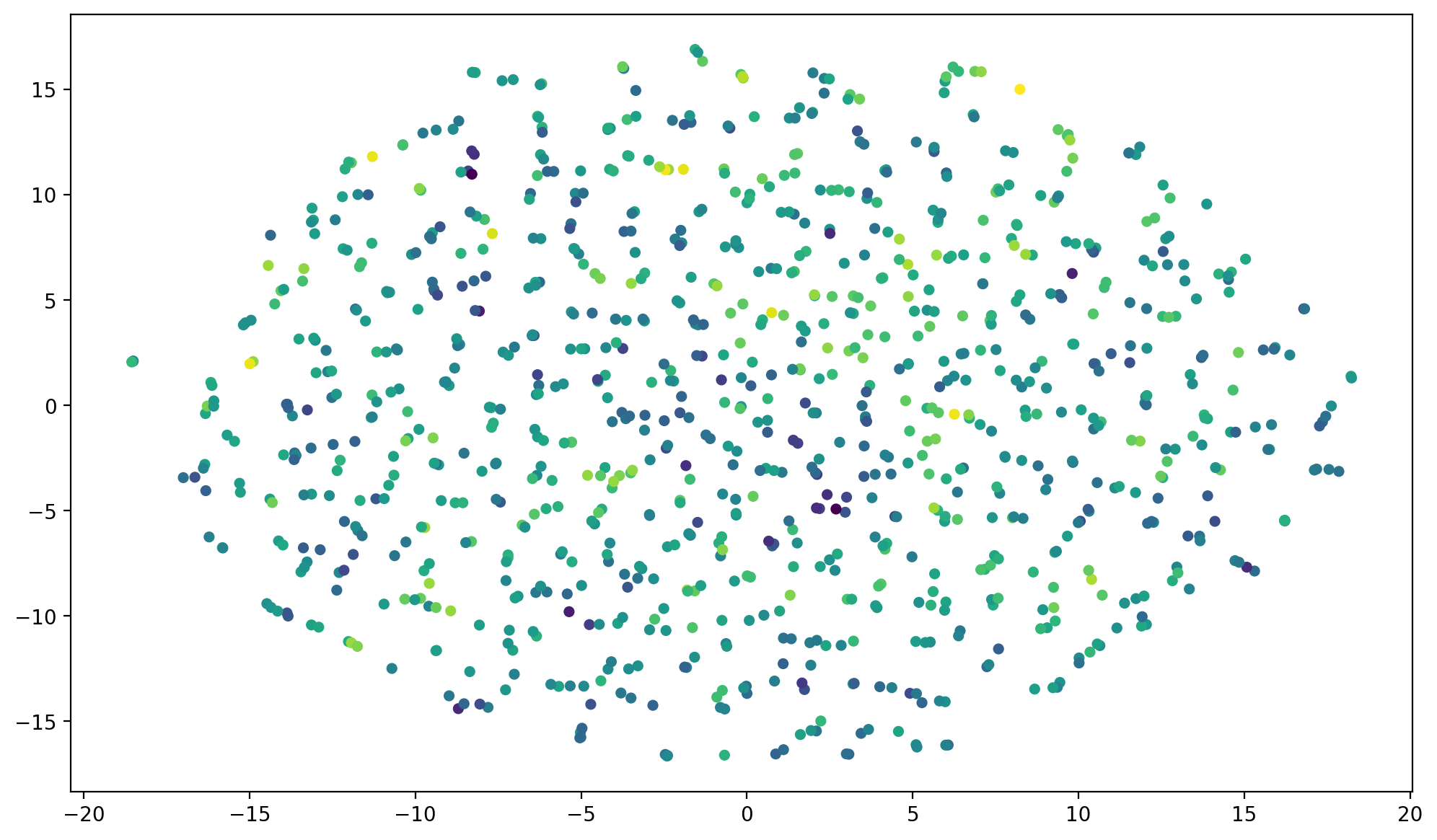}
        \caption{600 iterations}
    \end{subfigure}
    \hfill
    \begin{subfigure}[b]{0.3\textwidth}
        \centering
        \includegraphics[width=\textwidth]{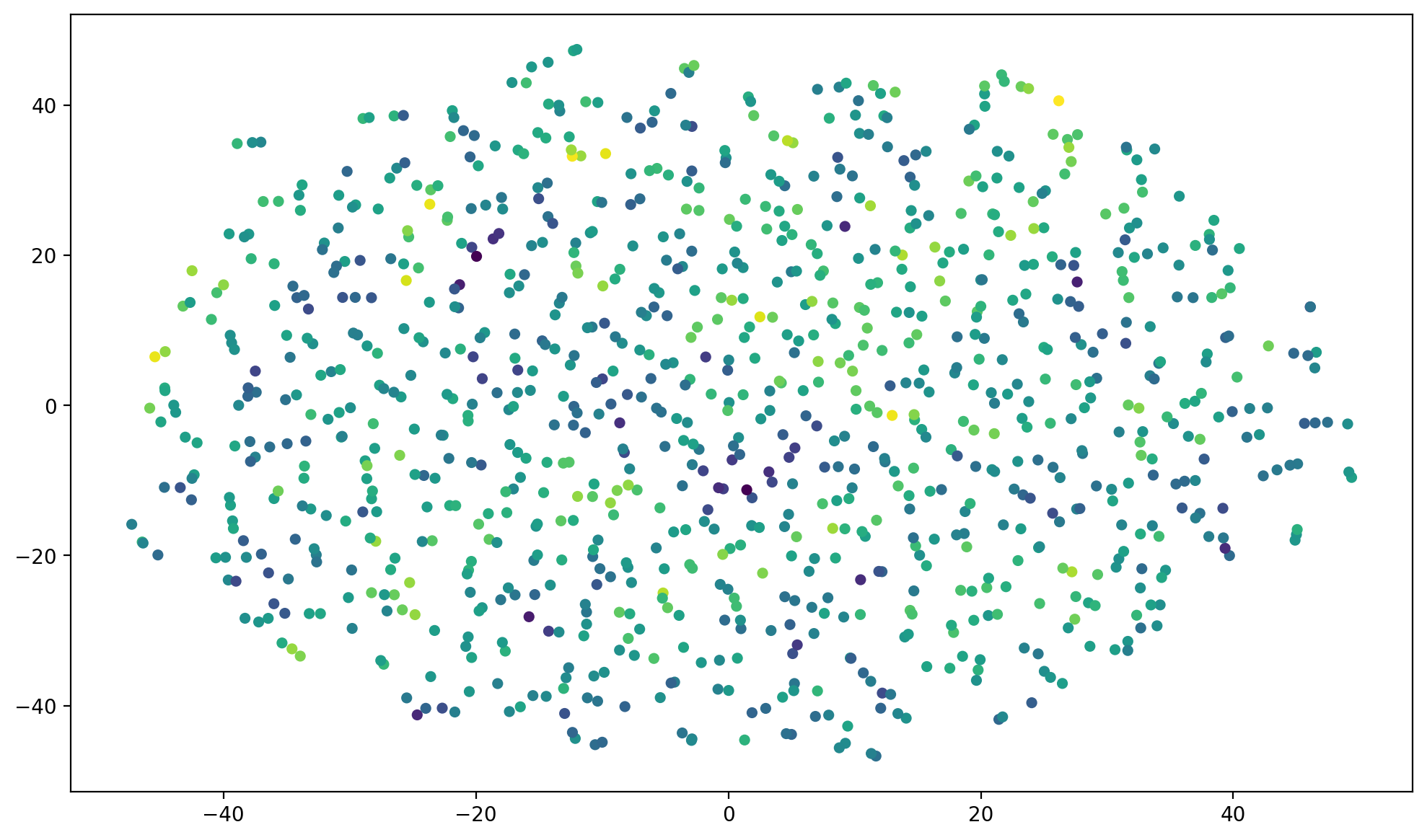}
        \caption{1000 iterations}
    \end{subfigure}
    \caption{t-SNE on points drawn from the unit sphere $\SS^{19}\subset\R^{20}$.}
    \label{fig:sphere_20}
\end{figure}
When we increase the dimension to $d=20$, plotted in \cref{fig:sphere_20}, the final visualization is even worse than it was in the case $d=5$.
We removed the panels showing the first two coordinates and the initialization, as they are very similar to the previous examples, and added panels for the intermediate results after 100 and 600 iterations to better illustrate the dynamics of t-SNE, which are particularly interesting in this example.
The scale of the plots from 10 to 100 to 500 iterations rapidly decreases, behavior that was not seen in previous examples.
This aligns with the predictions of \cref{theorem:sphere}, i.e., t-SNE places all points very close together.
We note that after the early exaggeration phase, in this example the points reinflate to a larger scale; while \cref{theorem:sphere} does not apply to the t-SNE output during early exaggeration, because the early exaggeration gradient updates do not correspond to an objective function, one should heuristically expect \cref{theorem:sphere} to apply with an even smaller containing ball when using early exaggeration.
This is because the gradient argument used in the proof relies on showing that a far away point is unstable, namely with stronger attractive forces than repulsive forces, and early exaggeration amplifies the attractive forces.
This suggests $d=20$ is in an intermediate zone: too small for the effects of \cref{theorem:sphere} to be seen without early exaggeration but large enough for the small ball prediction to hold with early exaggeration.

\begin{figure}[htbp]
    \centering
    \begin{subfigure}[b]{0.3\textwidth}
        \centering
        \includegraphics[width=\textwidth]{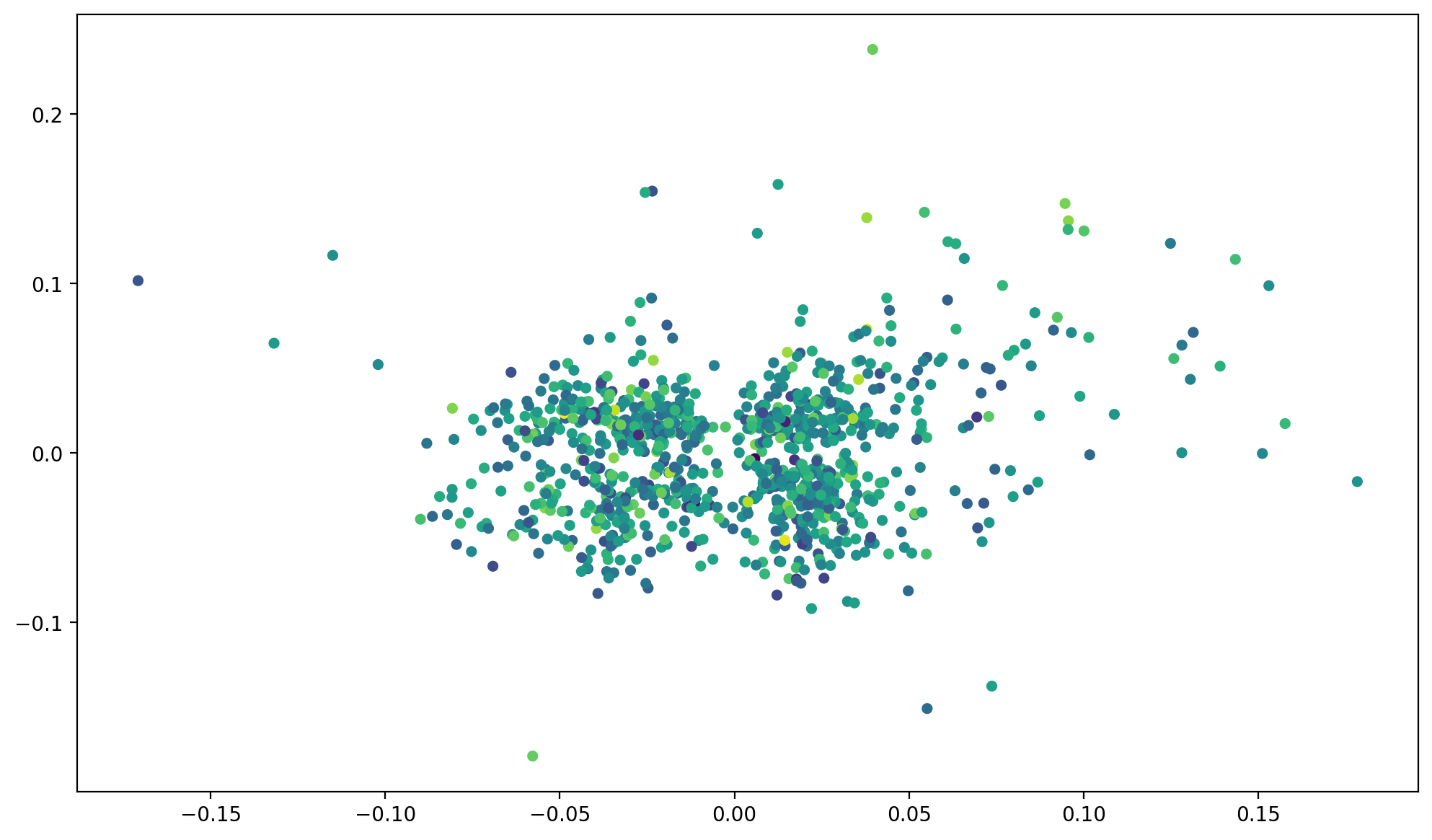}
        \caption{10 iterations}
    \end{subfigure}
    \hfill
    \begin{subfigure}[b]{0.3\textwidth}
        \centering
        \includegraphics[width=\textwidth]{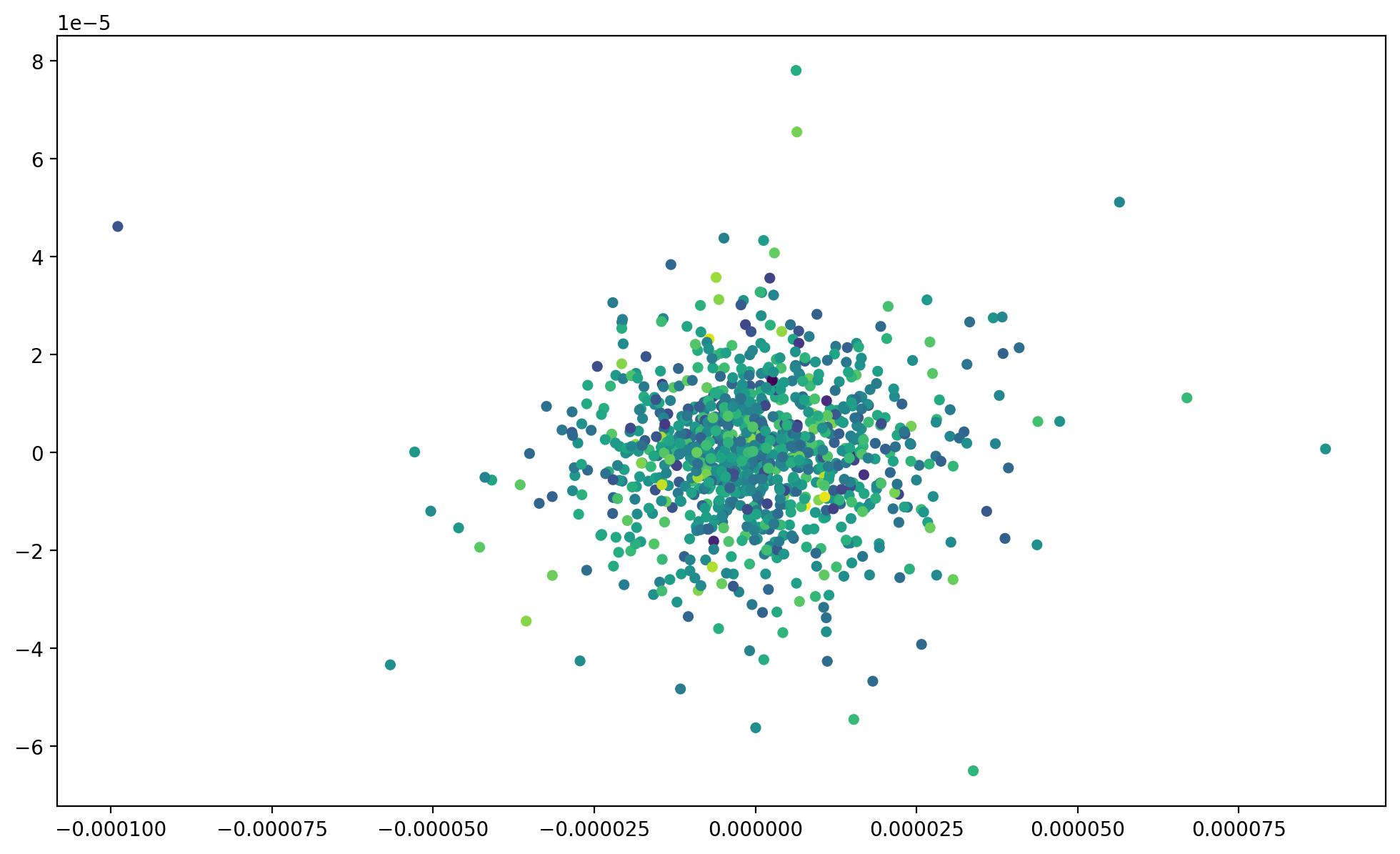}
        \caption{100 iterations}
    \end{subfigure}
    \hfill
    \begin{subfigure}[b]{0.3\textwidth}
        \centering
        \includegraphics[width=\textwidth]{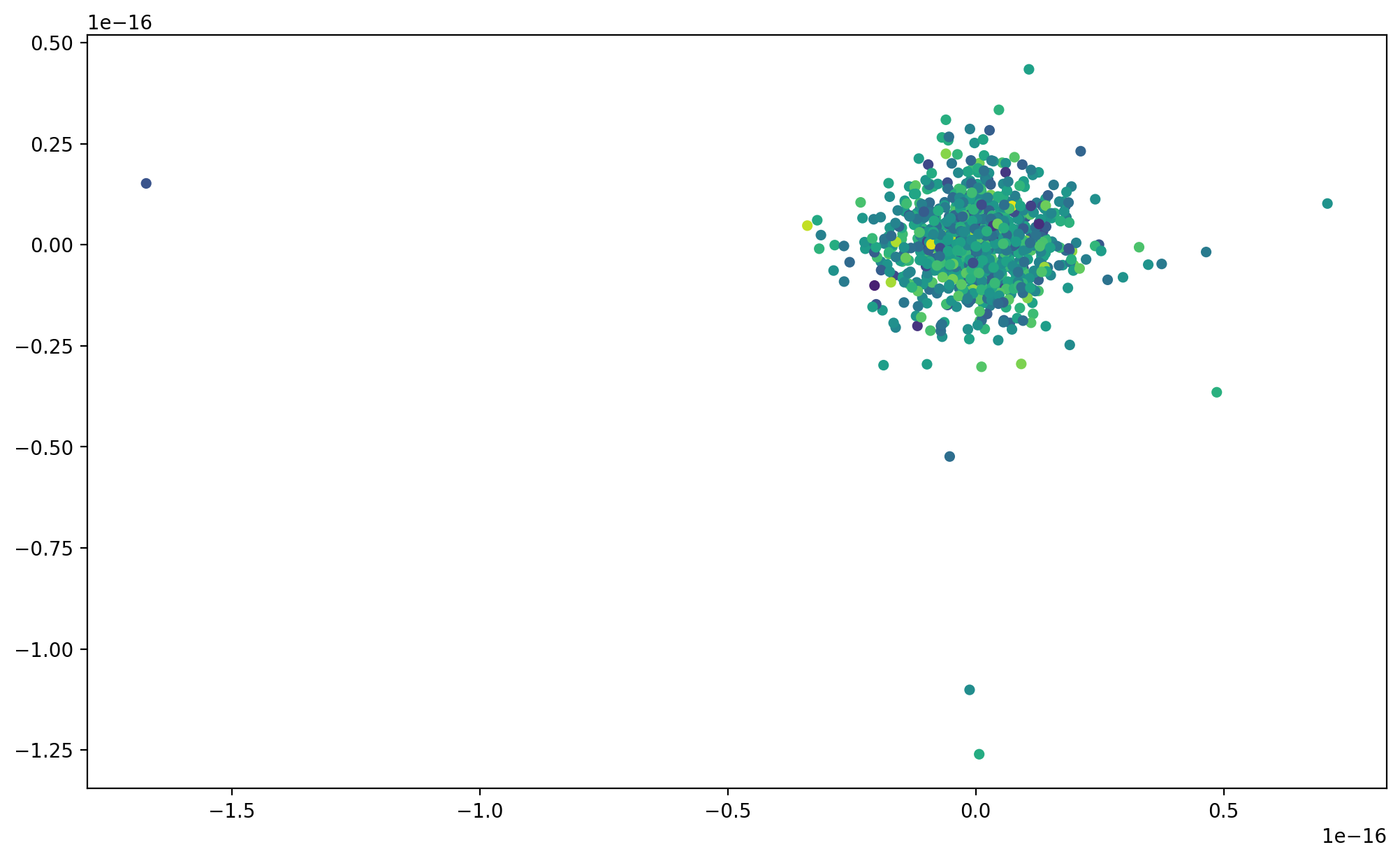}
        \caption{500 iterations}
    \end{subfigure}
    \begin{subfigure}[b]{0.3\textwidth}
        \centering
        \includegraphics[width=\textwidth]{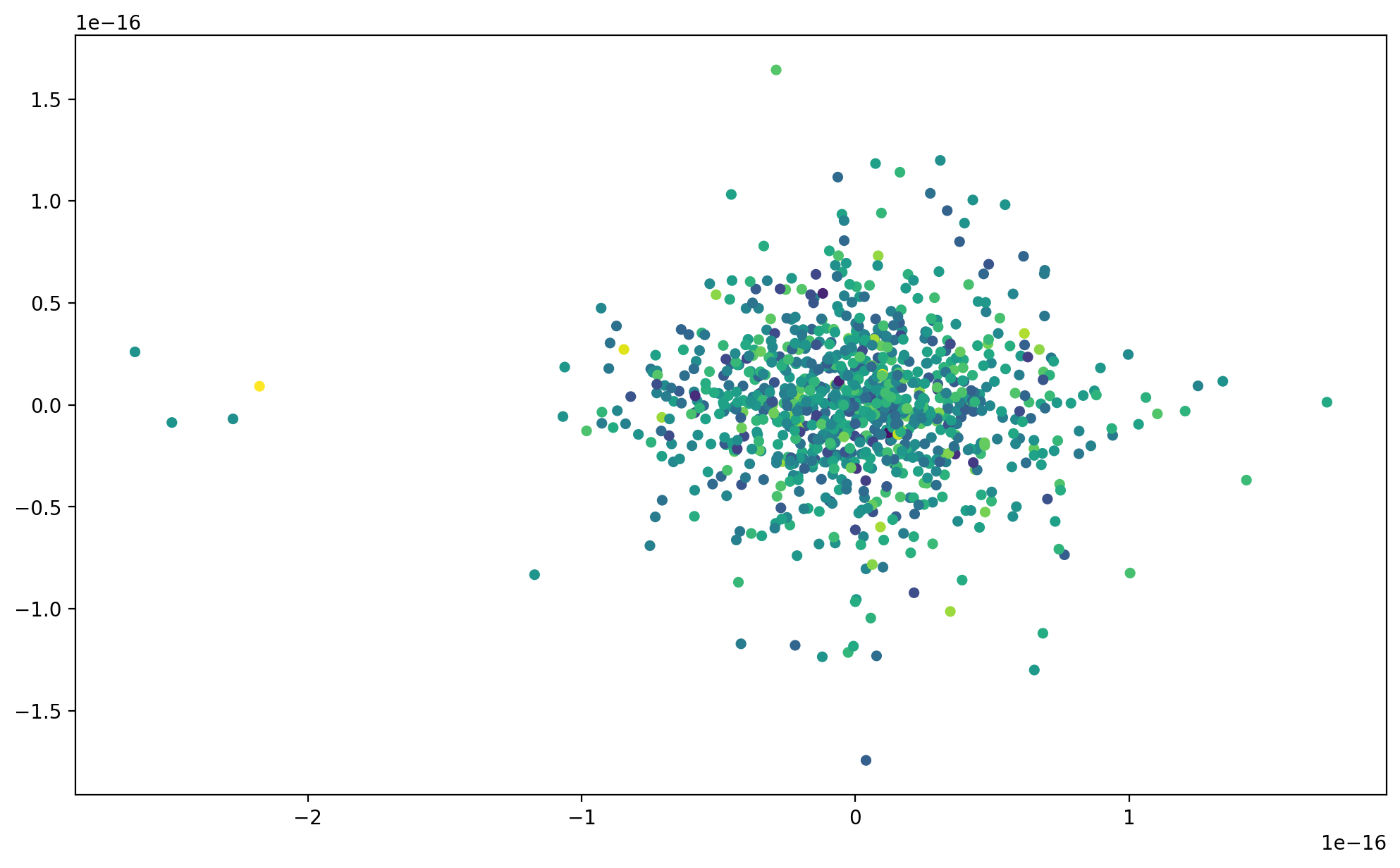}
        \caption{510 iterations}
    \end{subfigure}
    \hfill
    \begin{subfigure}[b]{0.3\textwidth}
        \centering
        \includegraphics[width=\textwidth]{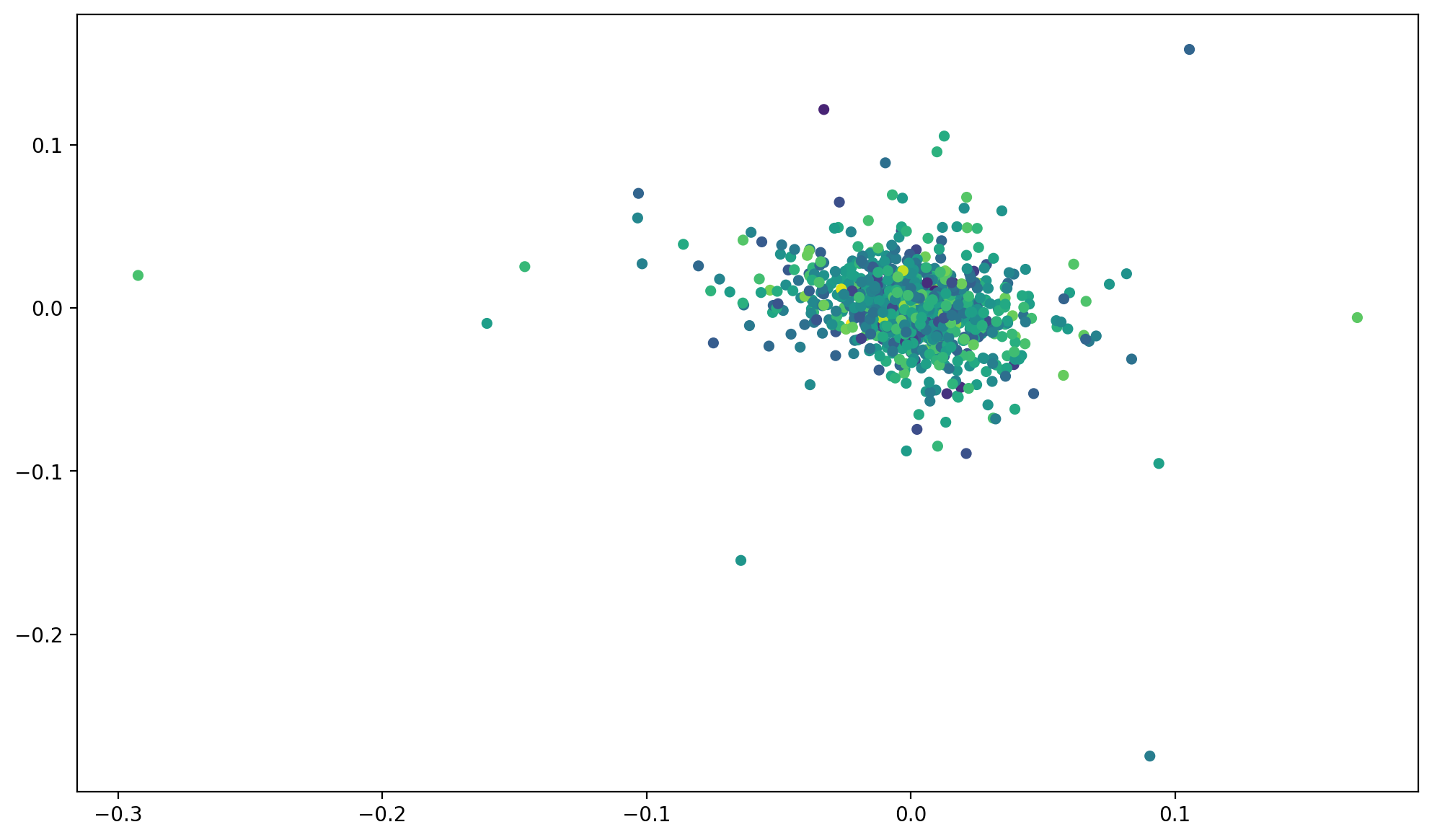}
        \caption{600 iterations}
    \end{subfigure}
    \hfill
    \begin{subfigure}[b]{0.3\textwidth}
        \centering
        \includegraphics[width=\textwidth]{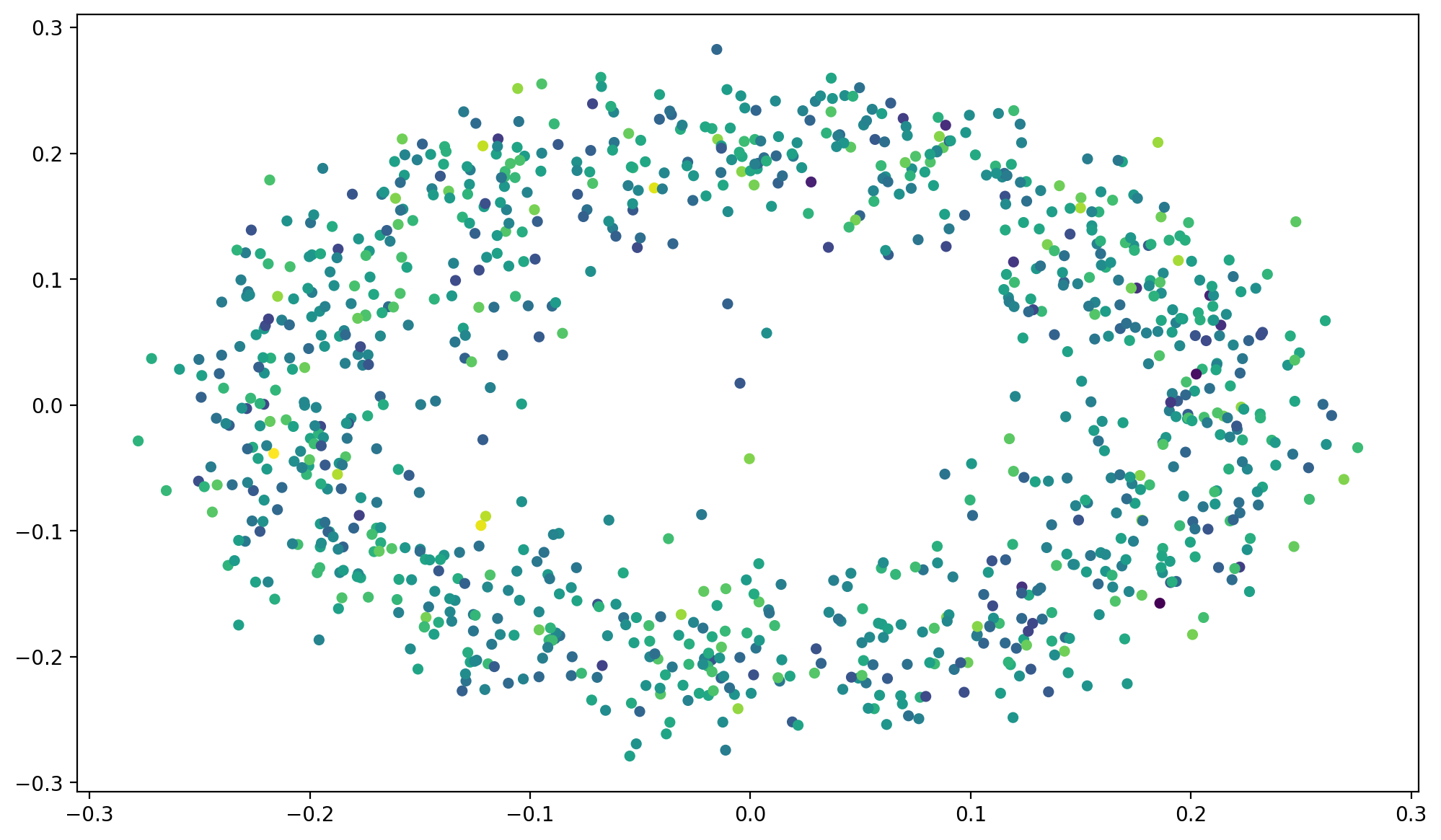}
        \caption{1000 iterations}
    \end{subfigure}
    \caption{t-SNE on points drawn from the unit sphere $\SS^{99999}\subset\R^{100000}$.}
    \label{fig:sphere_100K}
\end{figure}
Recall that the behavior expected from \cref{theorem:sphere} relies on the fact that i.i.d.\ uniformly random points on a high-dimensional sphere are approximately equidistant, thus yielding a near-uniform $P$.
Formally, our result assumes constant bandwidth $\sigma_i$ to make this implication from equidistant points to near-uniform $P$.
However, because we use default parameters in our numerical examples, where in particular perplexity is set to 30, this weakens this implication, as with sufficiently high dimension and more than 30 points, the low perplexity value will cause the bandwidths $\sigma_i$ to shrink to exaggerate the minor differences in the pairwise distances.
Thus, when not assisted by early exaggeration, our numerical examples need much larger dimension in order to see the effects from \cref{theorem:sphere}; \cref{fig:sphere_100K} shows the case $d=100,000$, where we can see the early exaggeration output is on an even smaller scale than in the case $d=20$, and the final output is also on a smaller scale than in the case $d=20$.
\subsection{Split sphere}
We now look at a numerical example of the split sphere model from \cref{section:split_sphere}.
We consider 1000 points drawn i.i.d.\ from the uniform distribution on the unit sphere $\SS^{d-1}$ for $d=20$, conditioned on the first coordinate having magnitude at least $d^{-0.1}$.
The results are plotted in \cref{fig:split_sphere_20}.
\begin{figure}[htbp]
    \centering
    \begin{subfigure}[t]{0.3\textwidth}
        \centering
        \includegraphics[width=\textwidth]{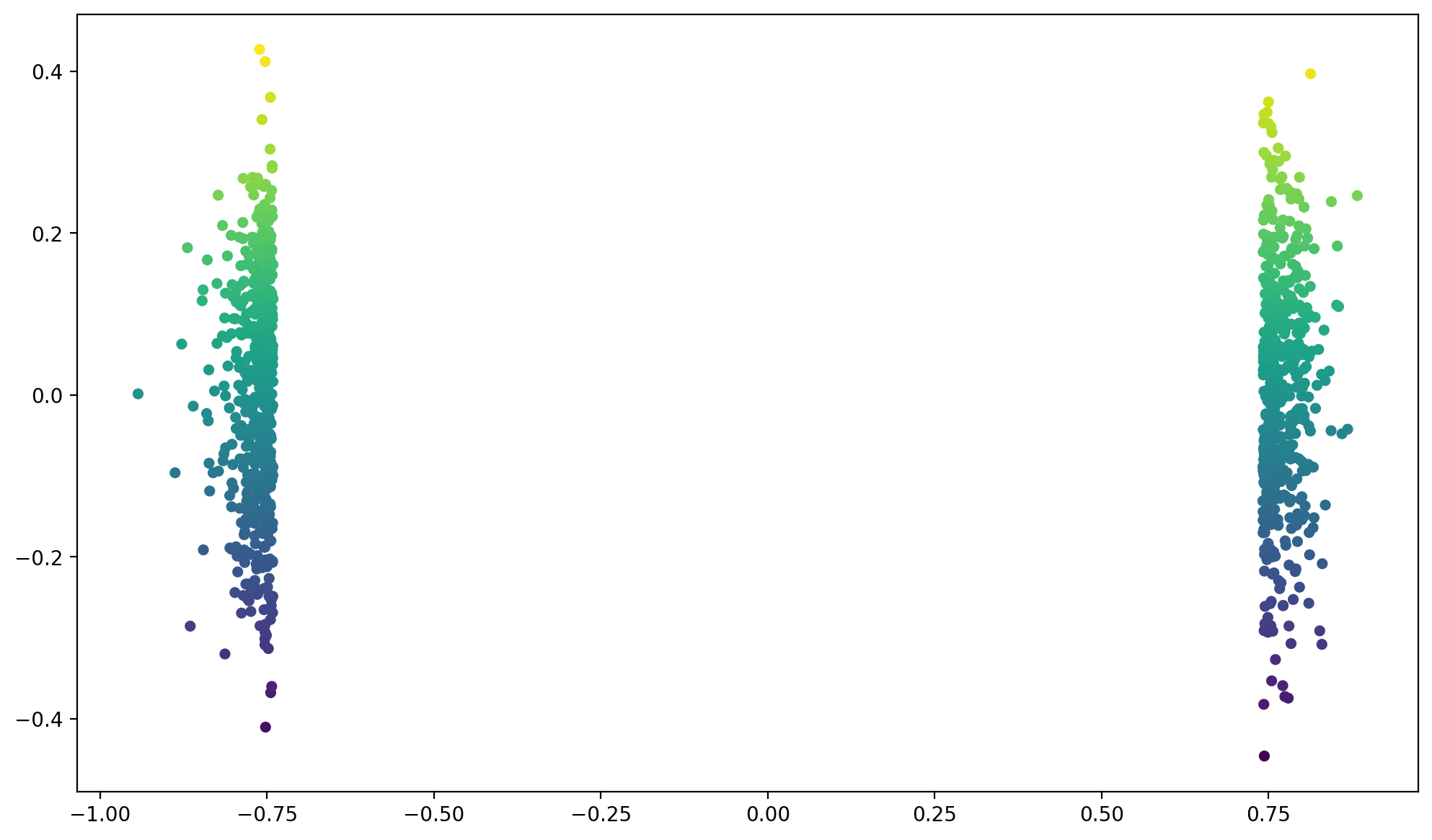}
        \caption{First two coordinates, colored by second coordinate}
    \end{subfigure}
    \hfill
    \begin{subfigure}[t]{0.3\textwidth}
        \centering
        \includegraphics[width=\textwidth]{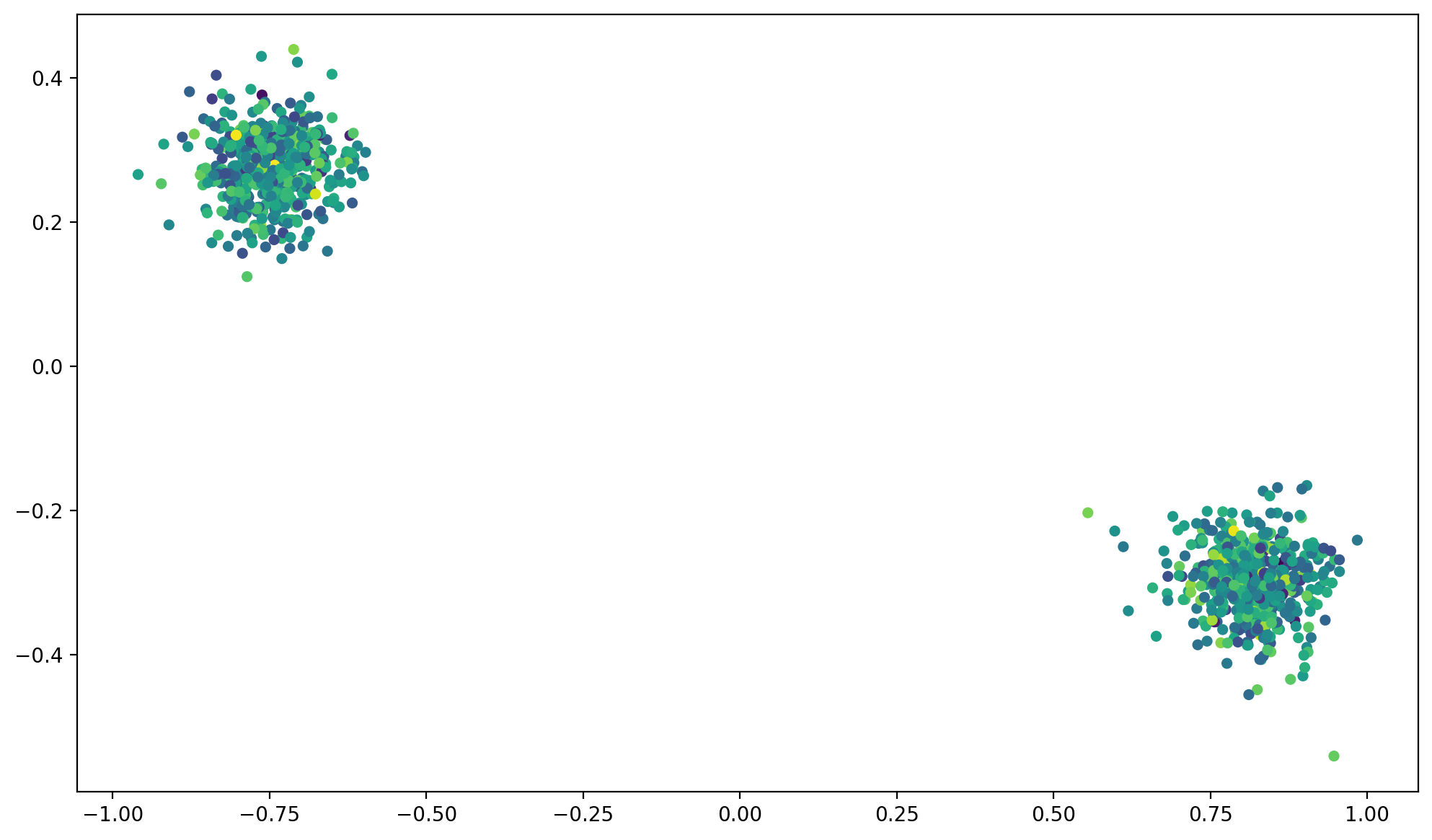}
        \caption{10 iterations}
    \end{subfigure}
    \hfill
    \begin{subfigure}[t]{0.3\textwidth}
        \centering
        \includegraphics[width=\textwidth]{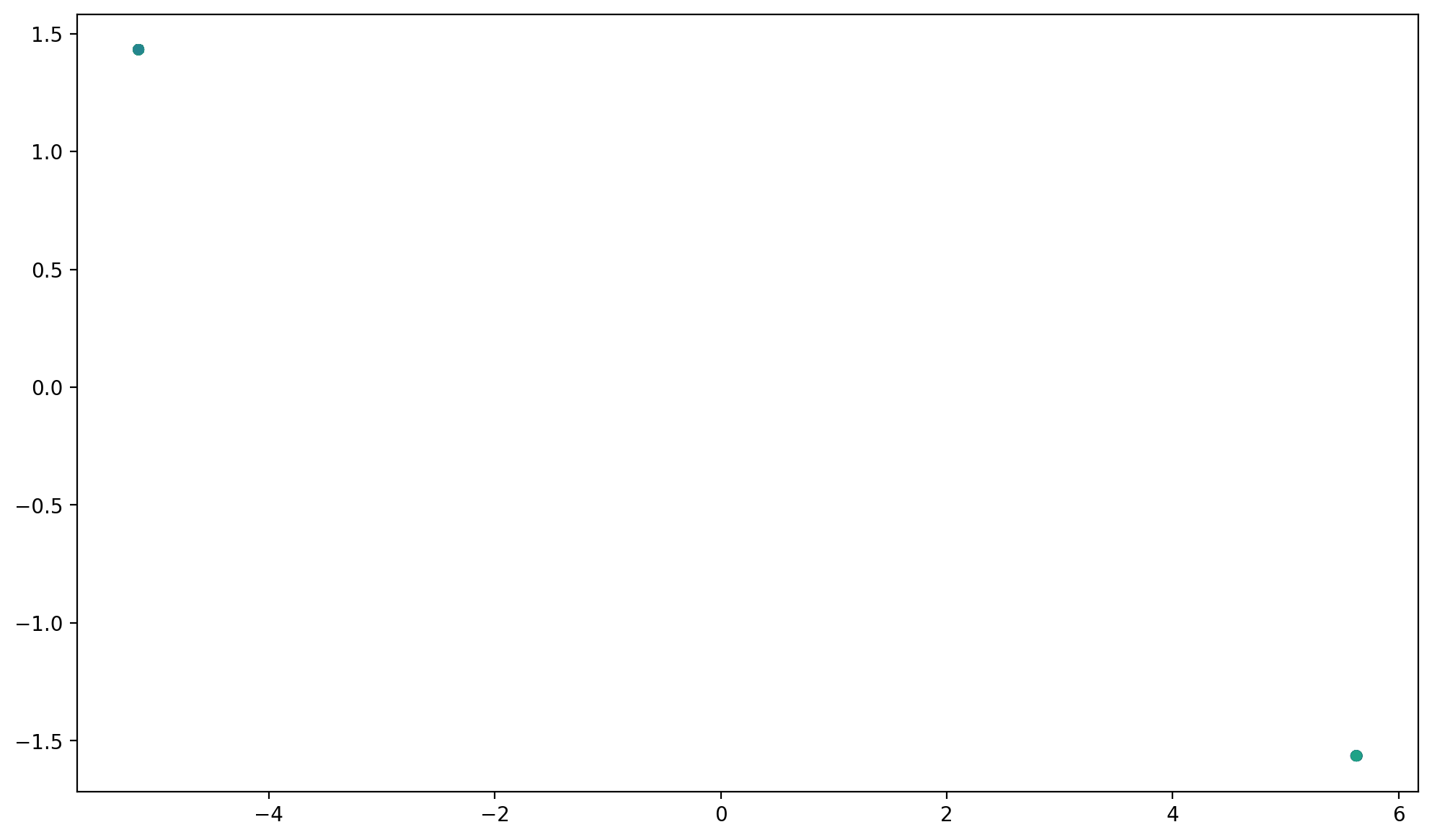}
        \caption{500 iterations}
    \end{subfigure}
    \begin{subfigure}[t]{0.3\textwidth}
        \centering
        \includegraphics[width=\textwidth]{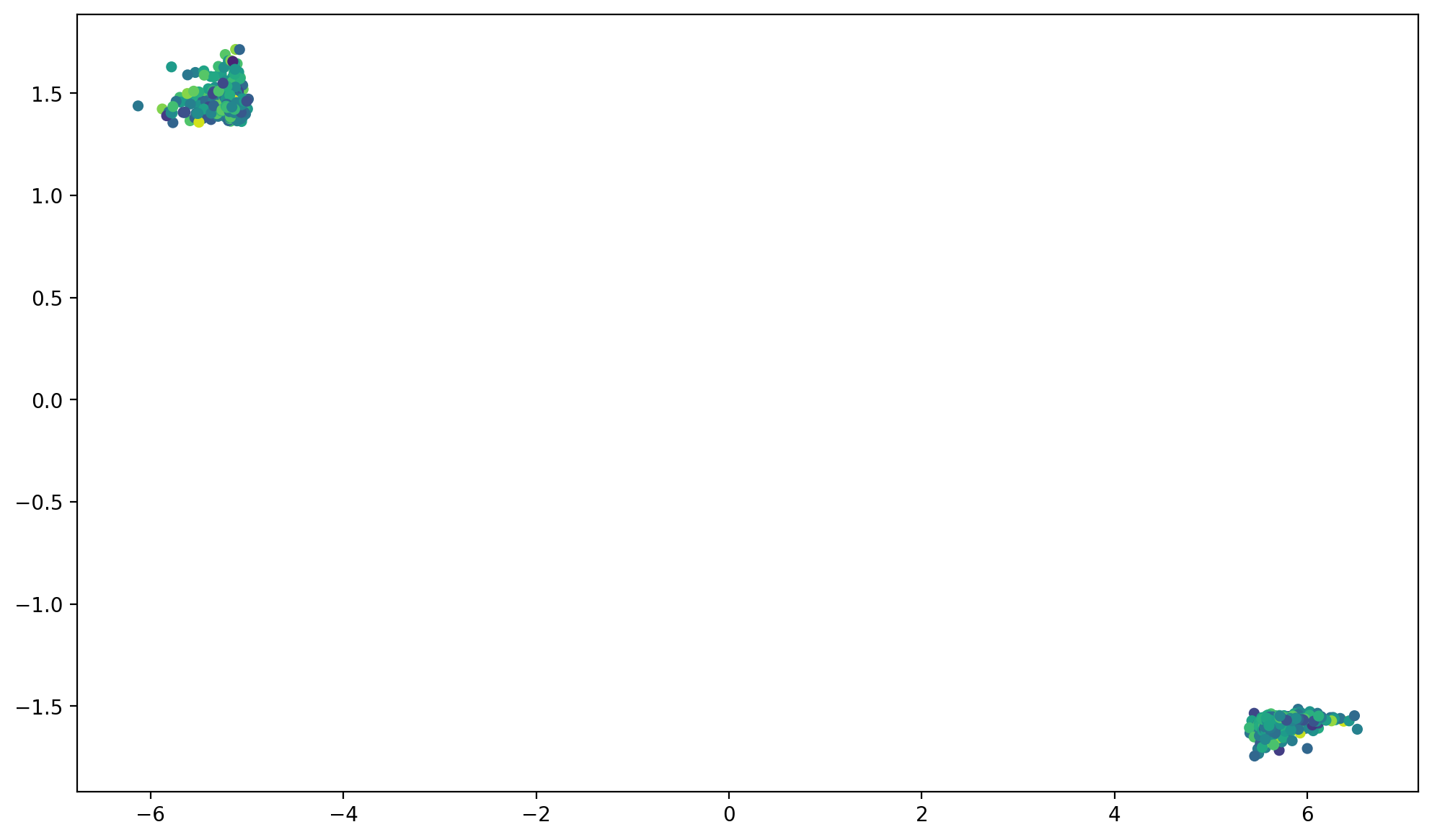}
        \caption{510 iterations}
    \end{subfigure}
    \hfill
    \begin{subfigure}[t]{0.3\textwidth}
        \centering
        \includegraphics[width=\textwidth]{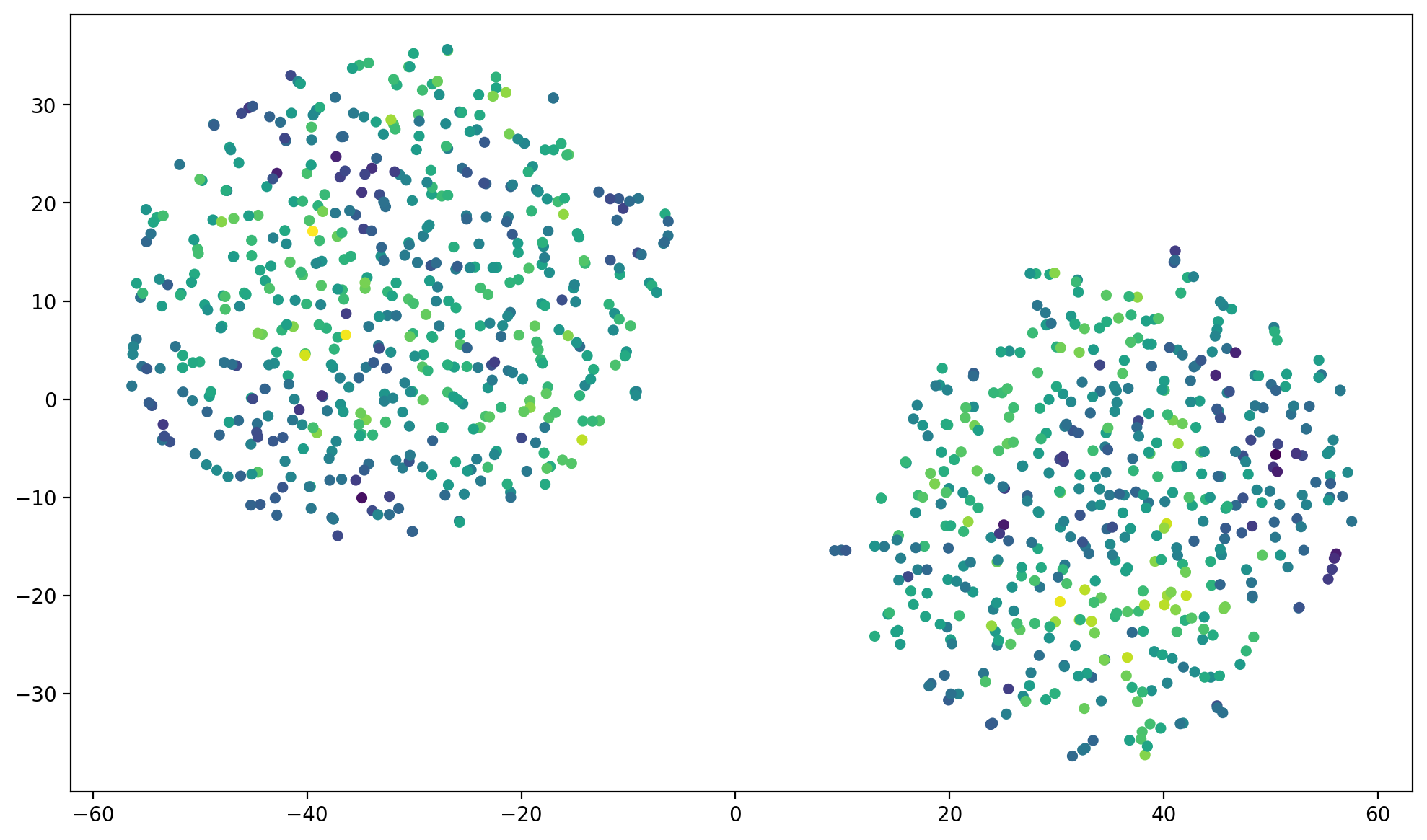}
        \caption{1000 iterations}
    \end{subfigure}
    \hfill
    \begin{subfigure}[t]{0.3\textwidth}
        \centering
        \includegraphics[width=\textwidth]{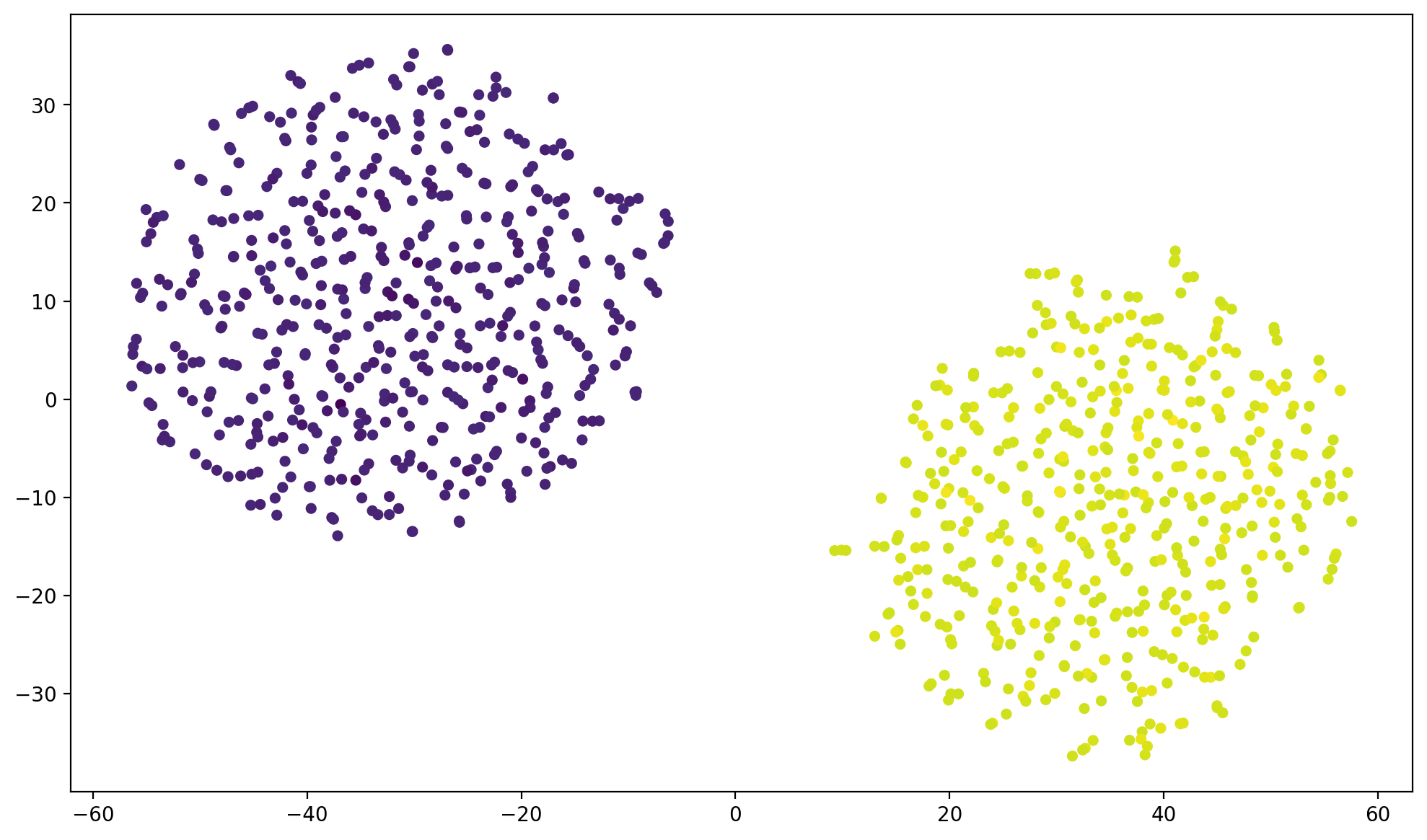}
        \caption{Final output colored by first coordinate}
    \end{subfigure}
    \caption{t-SNE on 1000 points drawn from the unit sphere $\SS^{19}\subset\R^{20}$ conditioned on the first coordinate having magnitude at least $20^{-0.1}\approx0.74$.}
    \label{fig:split_sphere_20}
\end{figure}
We see that t-SNE successfully identifies the two spherical caps and visualizes them as disjoint clusters, but within each cluster fails to capture any structural information, especially any local structure, as demonstrated when we color the points by their second coordinate.
This aligns with our first interpretation from \cref{section:introduction}.
The failure to capture local structure looks essentially similar to the t-SNE output for the entire unit sphere $\SS^{19}$ from \cref{fig:sphere_20}.
Similar to our discussion in \cref{subsec:sphere_examples}, we see that the effects of \cref{theorem:sphere} on each of the two clusters, namely putting all of the high-dimensional points very close together, are stronger with early exaggeration, i.e., after 500 iterations in our setup, than without, i.e., after 1000 iterations.
And similarly, we know from \cref{section:split_sphere} that each of the two clusters, as well as the overall scale of the two clusters combined, will shrink to zero as the dimension increases, as we previously began to see in \cref{fig:sphere_100K}.

\section*{Acknowledgments}
Rupert Li was partially supported by a Hertz Fellowship and a PD Soros Fellowship.
Elchanan Mossel was partially supported by NSF DMS-2031883, Bush Faculty Fellowship ONR-N00014-20-1-2826, and Simons Investigator award (622132).

We wish to thank Tomasz Mrowka for helpful discussions.

\bibliographystyle{amsinit}
\bibliography{ref}

\end{document}